\documentclass[letterpaper,journal]{IEEEtran}

\usepackage{amsmath,amsfonts}
\usepackage{algorithmic}
\usepackage{array}
\usepackage[caption=false,font=normalsize,labelfont=sf,textfont=sf]{subfig}
\usepackage{textcomp}
\usepackage{stfloats}
\usepackage{url}
\usepackage{verbatim}
\usepackage{graphicx}

\usepackage{cite}

\def\BibTeX{{\rm B\kern-.05em{\sc i\kern-.025em b}\kern-.08em
    T\kern-.1667em\lower.7ex\hbox{E}\kern-.125emX}}
\usepackage{balance}

\usepackage{xcolor}

\usepackage{amsthm}

\newtheorem{theorem}{Theorem}
\newtheorem{proposition}{Proposition}
\newtheorem{lemma}{Lemma}

\newtheorem{assumption}{Assumption}

\theoremstyle{definition}
\newtheorem{definition}{Definition}
\newtheorem{exampleinner}{Example}

\newenvironment{example}
  {\pushQED{\qed}\exampleinner}
  {\popQED\endexampleinner}

\newtheorem{remark}{Remark}

\usepackage{bm}

\newcommand{\E}{\mathbb{E}}

\newcommand{\bx}{\bm{x}}
\newcommand{\by}{\bm{y}}
\newcommand{\bz}{\bm{z}}
\newcommand{\bu}{\bm{u}}

\newcommand{\reals}{\mathbb{R}}

\definecolor{mycyan}{RGB}{0,204,204}

\definecolor{myorange}{RGB}{252, 105, 16}

\begin{document}
\title{Everywhere Learning: Artificial Intelligence with Pointwise Constraints}
\author{Ignacio~Boero*,
        Ignacio~Hounie*,
        Luiz~F. O. Chamon,
        and Alejandro~Ribeiro%
\thanks{Ignacio Boero (iboero@seas.upenn.edu) is with the Department of Electrical and Systems Engineering, University of Pennsylvania.}
\thanks{Ignacio Hounie (ihounie@seas.upenn.edu) is with the Department of Electrical and Systems Engineering, University of Pennsylvania.}
\thanks{Luiz F. O. Chamon (luiz.chamon@polytechnique.edu) is with the École polytechnique, Institut Polytechnique de Paris.}

\thanks{Alejandro Ribeiro (aribeiro@seas.upenn.edu) is with the Department of Electrical and Systems Engineering, University of Pennsylvania.}

\thanks{* Corresponding authors}}


\maketitle

\begin{abstract}
Everywhere Learning is a new paradigm in which artificial intelligence (AI) systems are trained to satisfy loss constraints with probability one over the data distribution, in contrast to the standard paradigm of minimizing average loss. We develop a duality-based learning theory and establish generalization guarantees that bound the distance between the empirical and statistical optimal values. Our analysis shows that the dual variables induce a reweighting of the data distribution that concentrates the mass on points where constraints are the hardest to satisfy. Generalization is then governed by the mismatch between the original data distribution and this dual-induced distribution. We further show that this mismatch—and therefore the generalization error—can be controlled through an L1 penalty on constraint relaxations. We illustrate these mechanisms
on an agentic-workflow classification problem for language model
coding tasks, where pointwise training reduces the upper tail of
the loss distribution and improves performance uniformity across
data sources. 
\end{abstract}


\section{Introduction}
We study the learnability of \emph{everywhere learning} problems, a paradigm in which AI systems are trained to satisfy constraints with probability one over the data distribution. This paradigm departs from standard statistical risk minimization (SRM), which trains models by minimizing the expected value of a loss~\cite{shalev2014understanding}. Although SRM is effective for optimizing a single objective, it is not suitable for modern AI systems, whose deployment requires satisfying multiple, often competing, design criteria~\cite{Stewart2008, RobLear, ecommerce-example}. Everywhere learning separates objectives---what we want a model to optimize---from requirements---what the model needs to satisfy. This prevents trade-offs in which violations of one requirement are compensated by improvements in other requirements, as occurs in many extensions of SRM that aggregate multiple losses through weighted sums~\cite{sener2018multi, zhang2021survey, NEURIPS2020_ebea2325, multiobj, multilear}.

 Existing work has studied constrained learning formulations in which requirements are imposed in expectation over the data distribution~\cite{zhang2025alignmentlargelanguagemodels,pmlr-v84-narasimhan18a, NEURIPS2025_1af991de, chamon2022constrained-TIT, elenter2024nearoptimalsolutionsconstrainedlearning, boero2025alcoleaugmentedlagrangianconstrained}. 
 However, enforcing constraints only in expectation provides no guarantee on how often violations occur, nor on how large they can be at individual data points. A constraint may be satisfied on average, while certain regions of the input space systematically fail to meet the requirement. These regions may correspond to specific users, operating regimes, or safety-critical scenarios \cite{hashimoto2018fairness}. The issue becomes even more pronounced with multiple constraints, as different regions may violate different requirements, so all constraints can be satisfied in expectation, even though no region satisfies them all simultaneously. This motivates imposing constraints pointwise, enforcing each requirement to hold almost everywhere over the corresponding data distribution.

Everywhere learning, like any statistical learning problem, assumes that data distributions are unknown and are accessed only through samples. We therefore need to understand when the statistical problem can be approximated, with high probability, from finitely many samples. To this end, we introduce a probably approximately correct constrained (PACC) notion of learnability for everywhere learning (Definition \ref{def:pacc}). This definition extends classical PAC learnability~\cite{ Vapnik1999-VAPTNO} by requiring both approximate optimality of the objective value and approximate feasibility of the pointwise constraints. However, unlike in standard SRM, uniform convergence of the objective and constraint losses is not sufficient to guarantee a learnability convergence rate (Example~\ref{ex:counter-uc-lambda}). This motivates the study of learnability for everywhere learning problems. To address this, we turn to the dual domain.

We study when the dual problem can itself be solved from finite samples, a property we call \emph{dual learnability} (Definition \ref{def:dual-pac}). We then show how dual learnability transfers back to PACC learnability of the primal problem (Definition \ref{def:pacc}). The dual formulation introduces a multiplier \(\lambda(\bz)\) that reweights the constraint violation at each data point \(\bz\). We show that a sufficient condition for dual learnability is that the optimal multiplier has bounded infinity norm (Theorem \ref{thm:dual-learn}). This condition has a natural interpretation. The multiplier \(\lambda(\bz)\) measures the sensitivity of the optimum value to perturbations of the constraint at \(\bz\) relative to the probability of observing that sample. Large values of \(\lambda(\bz)\) therefore identify regions that have high impact on the optimum but are unlikely to appear in a finite sample set (Subsection \ref{sec:ass-disc}).

Although natural, the boundedness condition on \(\lambda^\star\) need not hold in practice. The constrained problem may be highly sensitive to regions that have small probability under the data distribution. We therefore enforce the condition directly by restricting the dual variables to have bounded \(L_\infty\) norm. This restriction has a simple primal interpretation. It is equivalent to introducing an \(L_1\)-penalized relaxation of the pointwise constraints, which encourages the relaxation to be sparse (Proposition \ref{prop:duality_res}). This yields a principled relaxation strategy in which better generalization is obtained by sparsely relaxing the pointwise constraints and modifying the original problem only where necessary. Moreover, this relaxation is simple to implement in practice because it amounts to clipping the empirical dual variables (Theorem~\ref{thm:dual_learnability_res}).

We then relate dual learnability back to primal learnability. In convex problems, strong duality allows dual guarantees to transfer directly to PACC learnability (Theorem~\ref{thm:pacc}), and solving either the empirical primal problem or the empirical dual problem yields the same statistical guarantees (Proposition~\ref{prop:primal-est-same-dual}). Modern AI parametrizations, however, are typically nonconvex. For this setting, we show that primal learnability can still be recovered up to an approximation term controlled by the richness of the hypothesis class. In particular, when the hypothesis class approximates its convex hull well, the duality gap is small and the empirical primal rule is near-PACC learnable (Theorem~\ref{thm:near-pacc}). For the empirical dual rule, the guarantee requires additional care because primal recovery is no longer automatic. We prove that there exists a Lagrangian minimizer satisfying near-PACC guarantees, with an additional term depending on the empirical complementary-slackness error (Proposition~\ref{prop:dual_near_pacc}).

We conclude with numerical experiments that validate the practical benefits of everywhere learning. The results show that enforcing pointwise constraints can outperform standard average loss minimization, that the proposed dual regularization can improve generalization, and that dual methods recover high-quality solutions even in nonconvex settings.




\subsection{Related Work}

\textbf{Constrained Learning}.
The constrained learning literature studies learning frameworks in which predictors are trained to satisfy requirements instead of minimizing an objective. Although recent works have demonstrated the practical value of learning under pointwise constraints~\cite{owerko2025learningoptimalpowerflow, ramirez2025feasiblelearning, hounie2026samplecountssupervisedfinetuning}, the theoretical understanding of such formulations remains limited. In particular, the work in~\cite{chamon2021probablyapproximatelycorrectconstrained} establishes learning guarantees for both expected-value and \(\mathcal D\)-almost everywhere constraints, but under a weaker notion of learnability, whereas~\cite{chamon2022constrained-TIT} strengthens these guarantees only in the case of expected-value constraints. Our work builds on these results by extending the stronger learnability guarantees of~\cite{chamon2022constrained-TIT} to the setting of \(\mathcal D\)-almost everywhere constraints.

\textbf{Robust Optimization}.
Robust optimization studies decision-making under uncertainty by seeking solutions that remain feasible, and often near-optimal, for all realizations of uncertain parameters within a prescribed uncertainty set~\cite{ben2009robust,bertsimas2010theoryapplicationsrobustoptimization, tulabandhula2014robustoptimizationusingmachine}. At a high level, our framework shares with robust optimization the objective of controlling adverse outcomes. However, the distinction is that in robust optimization the uncertainty is exogenous and encoded through a user-specified uncertainty set. In our setting, the constraints are required to hold \(\mathcal D\)-almost everywhere for some unknown distribution, from which only samples are available. Accordingly, our approach raises questions about statistical learnability which are absent in robust optimization.

\noindent
\textbf{Chance Constraints}. 
Chance-constrained optimization studies statistical optimization problems in which constraints are enforced with high probability over the data distribution~\cite{charnes1959chance, shapiroSP, kalogerias2022strong}. This framework is natural in applications such as portfolio optimization, where feasibility is required at a prescribed confidence level and a small probability of violation is part of the model. A large body of work studies finite-sample approximations of such problems, including scenario approximation~\cite{campibook,Nemirovski2006, campi} and sample-average approximation~\cite{Kim2015AGT,SAA2}. Since the required feasibility level is strictly below one, these methods can exploit combinatorial arguments that control how many sampled constraints may be violated while still providing probabilistic guarantees. Our setting corresponds to the limiting case in which the violation probability is zero. This modeling distinction is natural for modern learning problems, where the goal is to learn representations that satisfy requirements throughout the data distribution, rather than only at a prescribed confidence level. The absence of an allowed violation probability requires turning to the dual domain, where the multipliers quantify the sensitivity of the problem to these regions and determine whether the problem is learnable.

\section{Almost Everywhere Learning}
\label{sec:ae-l}

Let \(\{\mathcal D_i\}_{i=0}^m\) be probability distributions over input-output pairs \((\mathbf{x},\mathbf{y}) \in \mathcal X \times \mathcal Y\). Given loss functions \(\ell_i:\mathcal Y\times\mathcal Y\to[0,B]\) and thresholds \(c_i\in\mathbb R\), for \(i=0,\ldots,m\), we seek a parametric predictor \(f_\theta:\mathcal X\to\mathcal Y\) in the hypothesis class $\mathcal H := \{f_\theta:\theta\in\Theta\subseteq\mathbb R^p\}$
that solves the \emph{everywhere learning problem}:
\begin{equation}\label{P:primal_statistical}\tag{P}
\begin{aligned}
	P^\star = \min_{f_\theta \in \mathcal{H}}&
		&&\E_{(\bx,y) \sim \mathcal D_0} \!\Big[ \ell_0\big( f_{\theta}(\bx),y \big) \Big]
	\\
	\text{s. to}& &&  \ell_i\big( f_{\theta}(\bx),y \big)  \leq c_i \quad \mathcal D_i\text{-a.e.}
\end{aligned}
\end{equation}
The objective in~\eqref{P:primal_statistical} minimizes the loss \(\ell_0\) in expectation under \(\mathcal D_0\), as in standard statistical risk minimization (SRM). Unlike standard SRM, however, problem~\eqref{P:primal_statistical} also imposes constraints on several losses, allowing multiple requirements to be enforced simultaneously. These constraints require each loss to remain below its tolerance \(c_i\), preventing improvements in the objective from compensating for unsatisfied constraints. This differs from extensions of SRM that combine multiple losses through a weighted sum in the objective~\cite{sener2018multi, zhang2021survey}. Moreover, problem~\eqref{P:primal_statistical} imposes the constraints \emph{pointwise}, rather than only in expectation. This prevents large violations on individual data points from being hidden by good average performance, making the formulation appropriate when constraint violations are unacceptable. We next illustrate~\eqref{P:primal_statistical} through several AI applications.

\noindent
\textbf{Reliability}. In critical settings, where mistakes can have severe consequences, performance must be controlled even on rare events. Minimizing only the average loss can lead to large disparities across inputs and arbitrarily poor performance on low-probability regions. We enforce reliability by requiring the objective loss to remain below a maximum acceptable error \(c\) almost everywhere:
\begin{equation}
\begin{aligned}
\min_{\theta}\quad & \E_{(\bx,y)\sim \mathcal D_0}\!\left[ \ell_0\big(f_\theta(\bx),y\big)\right] \\
\text{s.t.}\quad &  \ell_0\big(f_\theta(\bx),y\big) \leq c
\qquad \mathcal D_0\text{-a.e.}
\end{aligned}
\end{equation}

\noindent
\textbf{Invariance}. Many AI tasks require predictions to be invariant under specific input perturbations. Examples include rotations in image classification or changes to a person's gender in credit assignment. Let \(g(\bx)\) denote the perturbed input. We enforce invariance by constraining the discrepancy between the predictions at \(\bx\) and \(g(\bx)\):
\begin{equation}
\begin{aligned}
\min_{\theta}\quad & \E_{(\bx,y)\sim \mathcal D_0}\!\left[\ell_0(f_\theta(\bx),y)\right] \\
\text{s.t.}\quad & \|f_\theta(g(\bx)) - f_\theta(\bx)\|_2^2 \leq c
\qquad \mathcal D\text{-a.e.}
\end{aligned}
\end{equation}

\noindent
\textbf{Fine-tuning}. Fine-tuning consists of adapting a reference model \(\pi_{\mathrm{ref}}\) to a new task while remaining as close as possible to the reference model. For language models, we represent learning the new task \(\mathcal D_{\mathrm{new}}\) by requiring the model to assign probability at least \(c\) to a correct response. At the same time, we minimize the KL divergence on a reference distribution \(\mathcal D_{\mathrm{ref}}\) to limit deviation from \(\pi_{\mathrm{ref}}\):
\begin{equation}
\begin{aligned}
\min_{\theta}\quad &
\E_{(\bx,\by)\sim \mathcal D_{\mathrm{ref}}}\!\left[
D_{\mathrm{KL}}\big(\pi_\theta(\by \mid \bx)\,\|\,\pi_{\mathrm{ref}}(\by \mid \bx)\big)
\right] \\
\text{s.t.}\quad &
\pi_\theta(\by\mid \bx) \geq c
\qquad \mathcal D_{\mathrm{new}}\text{-a.e.}
\end{aligned}
\end{equation}

Henceforth, to reduce notational burden, we consider the single-constraint case by setting \(i=1\) and denoting \(\ell_1 - c_1:=\ell\). All results below extend directly to multiple constraints.


\subsection{Learnability of constrained problems}

Because the distributions in~\eqref{P:primal_statistical} are unknown, the problem must be solved from samples drawn from \(\mathcal D_0\) and \(\mathcal D\). To introduce the relevant learning-theoretic notions, we first consider the unconstrained counterpart of~\eqref{P:primal_statistical}, which corresponds to the standard statistical risk minimization (SRM) problem:
\begin{equation}\label{P:primal_unconstrained}
\begin{aligned}
	Q^\star = \min_{f_\theta \in \mathcal{H}}&
		&&\E_{(\bx,y) \sim \mathcal D_0} \!\Big[ \ell_0\big( f_{\theta}(\bx),y \big) \Big].
\end{aligned}
\end{equation}
Classical learning theory~\cite{Vapnik1999-VAPTNO} studies when~\eqref{P:primal_unconstrained} is learnable in the \emph{probably approximately correct} sense.

\begin{definition}[PAC learnability]\label{def:pac}
A hypothesis class \(\mathcal{H}\) is \emph{probably approximately correct} (PAC) learnable with respect to a loss function \(\ell_0\) if, for every \(\epsilon,\delta \in (0,1)\) and every distribution \(\mathcal D_0\), there exists a learning rule such that, given \(N \ge N_{\mathcal{H}}(\epsilon,\delta)\) samples, it returns an estimator \(f_{\theta} \in \mathcal{H}\) satisfying, with probability at least \(1-\delta\),
\begin{equation}\label{eq:pac_obj}
\left|
\mathbb{E}_{(\bx,y)\sim \mathcal D_0}\big[\ell_0(f_{\theta}(\bx),y)\big]
-
Q^\star
\right|
\le
\epsilon .
\end{equation}
Such a rule is called a \emph{PAC learner}.
\end{definition}

Definition~\ref{def:pac} requires a learning rule whose returned estimator achieves statistical risk close to the optimal value, with high probability over the sampled training set. A natural PAC learner for~\eqref{P:primal_unconstrained} is empirical risk minimization (ERM). Given an i.i.d.\ sample \(S_0=\{(\bx_i,\by_i)\}_{i=1}^{N_0}\) from \(\mathcal D_0\), the ERM problem is
\begin{equation}\label{P:primal_unconstrained_empirical}
\begin{aligned}
	\hat Q^\star = \min_{f_\theta \in \mathcal{H}}&
		&& \frac{1}{N_0}\sum_{i=1}^{N_0}
        \ell_0\big( f_{\theta}(\bx_i),\by_i \big).
\end{aligned}
\end{equation}

Problem~\eqref{P:primal_unconstrained_empirical} replaces the expectation in~\eqref{P:primal_unconstrained} with a sample average. For any fixed \(f_\theta\), the law of large numbers ensures that this empirical average converges to the corresponding statistical risk. However, this hypothesis-wise convergence is not sufficient for learnability. For any fixed sample set \(S_0\), there may exist predictors \(f_\theta\in\mathcal H\) whose empirical risk poorly approximates their statistical risk. Since ERM selects the predictor using the same sample set \(S_0\), it may choose such pathological predictors. Uniform convergence rules out this behavior by requiring the empirical and statistical risks to be close uniformly over the hypothesis class.

\begin{definition}[Uniform convergence]\label{def:uniform-conv}
A hypothesis class \(\mathcal H\) satisfies the uniform convergence property with respect to a loss function \(\ell\) if there exists a function \(\zeta(N,\delta)\geq 0\) monotonically decreasing in \(N\) to $0$, such that for every distribution \(\mathcal D\), with probability at least \(1-\delta\) over independent samples \(\{(\bx_n,\by_n)\}_{n=1}^N \sim \mathcal D^N\) satisfies
\begin{equation}\label{eq:unif-conv-obj}
\left|
\mathbb{E}_{(\bx,y)\sim \mathcal{D}}
\!\left[
\ell\!\big(f_\theta(\bx),y\big)
\right]
-
\frac{1}{N}\sum_{n=1}^N
\ell\!\big(f_\theta(\bx_n),y_n\big)
\right|
\leq
\zeta(N,\delta),
\end{equation}
uniformly over all \(f_\theta \in \mathcal H\).
\end{definition}

The uniform convergence property provides a rate \(\zeta(N,\delta)\) at which the empirical risk approximates the statistical risk uniformly over the hypothesis class. Although this is not explicit in Definition~\ref{def:uniform-conv}, requiring the bound to hold for all \(f_\theta \in \mathcal H\) imposes a complexity restriction on \(\mathcal H\). Appendix~\ref{apx:radam} in the supplementary material 
introduces the Rademacher complexity of a hypothesis class and show that a vanishing Rademacher complexity yields a rate of the form~\eqref{eq:unif-conv-obj}. Henceforth, we assume that the uniform convergence rates are obtained through vanishing Rademacher complexity. A central result in classical learning theory is that uniform convergence is sufficient for PAC learnability of~\eqref{P:primal_unconstrained}, with ERM serving as a PAC learner.

\begin{lemma}[UC implies PAC, Corollary 4.4 of~\cite{shalev2014understanding}]
\label{prop:unif-conv-imply-pac}
If a hypothesis class \(\mathcal H\) satisfies the uniform convergence property, then \(\mathcal H\) is PAC learnable. Moreover, the ERM rule is a PAC learner for \(\mathcal H\).
\end{lemma}
Extending PAC learnability to the constrained problem~\eqref{P:primal_statistical} requires imposing guarantees not only on the objective value, but also on constraint satisfaction. This motivates the notion of \emph{probably approximately correct constrained} learnability.

\begin{definition}[PACC learnability]\label{def:pacc}
A hypothesis class \(\mathcal{H}\) is \emph{probably approximately correct constrained} (PACC) learnable with respect to an objective loss \(\ell_0\) and a constraint loss \(\ell\) if, for every \(\epsilon,\delta \in (0,1)\) and every pair of distributions \(\mathcal D_0\) and \(\mathcal D\), there exists a learning rule such that, given \(N_0 \ge N_{\mathcal H}(\epsilon,\delta)\) samples from \(\mathcal D_0\) and \(N \ge N_{\mathcal H}(\epsilon,\delta)\) samples from \(\mathcal D\), it returns a predictor \(f_\theta \in \mathcal H\) satisfying, with probability at least \(1-\delta\),
\begin{enumerate}
\item \emph{Probably approximately optimal:}
\begin{equation}\label{eq:pacc_obj}
\left|
\mathbb{E}_{(\bx,y)\sim \mathcal D_0}
\big[
\ell_0(f_\theta(\bx),y)
\big]
-
P^\star
\right|
\leq
\epsilon .
\end{equation}

\item \emph{Probably approximately feasible:}
\begin{equation}\label{eq:pacc_constraint}
\mathbb{P}_{(\bx,y)\sim \mathcal D}
\big[
\ell(f_\theta(\bx),y) \leq 0
\big]
\geq
1-\epsilon .
\end{equation}
\end{enumerate}
\end{definition}

PACC learnability extends classical PAC learnability by adding a feasibility requirement. The optimality condition~\eqref{eq:pacc_obj} mirrors the usual PAC objective guarantee, while~\eqref{eq:pacc_constraint} requires the learned predictor to satisfy the constraint on a set of probability at least \(1-\epsilon\). A related notion of PACC learnability for pointwise constraints was introduced in~\cite{chamon2021probablyapproximatelycorrectconstrained}, where optimality is required only through an upper bound, rather than through the two-sided condition in~\eqref{eq:pacc_obj}. This distinction is relevant because every feasible point of~\eqref{P:primal_statistical} is feasible for~\eqref{P:primal_empirical}, since constraints that hold almost everywhere also hold on any finite sample set. Thus, \eqref{P:primal_empirical} relaxes the statistical constraint set and may attain lower objective values by exploiting this relaxation. The lower bound in~\eqref{eq:pacc_obj} is therefore necessary to certify convergence to the statistical constrained optimum. Further discussion is provided in Appendix~\ref{apx:emp_is_relax}.

A natural extension of the ERM problem to the everywhere learning is imposing the constraints for every sample, while replacing the expectation with the sample average. Let \(S_0=\{(\bx_i,y_i)\}_{i=1}^{N_0}\) and \(S=\{(\bx_i,y_i)\}_{i=1}^{N}\) be i.i.d.\ samples from \(\mathcal D_0\) and \(\mathcal D\), respectively. The constrained empirical problem is,
\begin{equation}\label{P:primal_empirical}\tag{$\hat{\textup{P}}$}
\begin{aligned}
	\hat P^\star
    =
    \min_{f_\theta \in \mathcal{H}}
    \quad
	&
    \frac{1}{N_0}
    \sum_{(\bx_i,y_i)\in S_0}
    \ell_0\big(f_\theta(\bx_i),y_i\big)
	\\
	\text{s.t.}
    \quad
    &
    \ell\big(f_\theta(\bx_i),y_i\big)
    \leq 0
    \quad
    \forall\,(\bx_i,y_i)\in S .
\end{aligned}
\end{equation}

Since \eqref{P:primal_empirical} is the direct analogue of ERM for constrained problems, it is natural to ask whether PACC learnability follows from uniform convergence of the objective loss, the constraint loss, and the indicator of constraint satisfaction \(\mathbb{I}\{\ell(f_\theta(\bx),y)\leq 0\}\).

\begin{assumption}[Uniform convergence]\label{ass:unif_conv}
The uniform convergence property in Definition~\ref{def:uniform-conv} holds for the losses \(\ell_0\), \(\ell\), and \(\mathbb{I}\{\ell \leq 0\}\), with rates
\(\zeta_0(N,\delta)\), \(\zeta(N,\delta)\), and \(\zeta_I(N,\delta)\), respectively.
\end{assumption}

Perhaps surprisingly, this assumption is not sufficient. In Subsection~\ref{sec:ass-disc}, Example~\ref{ex:counter-uc-lambda} shows a case where Assumption~\ref{ass:unif_conv} holds but no distribution free rate can be obtained, thus preventing PACC learnability. This motivates the need to understand when constrained problems are PACC learnable, and why PACC learnability can fail. To answer these questions, we turn to the dual domain.
\section{Dual learning}\label{sec:dual}
Let \(\bz=(\bx,y)\) be a shorthand notation of the data pair and let \(\mathcal{Z}=\mathcal{X}\times\mathcal{Y}\) be the range of $\bz$. We define a Lagrange multiplier as a nonnegative function\footnote{For ease of exposition, we represent dual variables as nonnegative functions. Formally, these functions are densities of finite nonnegative measures with respect to the data distribution. Appendix \ref{apx:distr_space} in the supplementary material
discusses the measure-theoretic formulation.} \(\lambda:\mathcal{Z}\to\mathbb{R}_+\), which assigns to each sample \(\bz\) a multiplier value \(\lambda(\bz)\geq 0\). We restrict multipliers to be absolutely integrable with respect to the data distribution, i.e., \(\lambda \in L^1_+(\mathcal D)\). The Lagrangian function \(L:\mathcal{H}\times L^1_+(\mathcal D)\to\mathbb{R}\) takes a predictor \(f_\theta\) and a multiplier \(\lambda\), and adds to the objective loss a multiplier-weighted penalty for constraint violations,
\begin{equation}\label{E:lagrangian_csl_param}
	L(f_{\theta}, \lambda)
	:=
	\E_{\bz \sim \mathcal D_0}
	\!\Big[
	    \ell_0\big(f_{\theta},\bz\big)
	\Big]
	+
	\E_{\bz \sim \mathcal D}
	\!\Big[
	    \lambda(\bz)\ell\big(f_{\theta},\bz\big)
	\Big].
\end{equation}
In \eqref{E:lagrangian_csl_param} we write \(\ell(f_\theta(\bx),y) = \ell(f_\theta,\bz)\) for notation simplicity. 
The minimizer of the Lagrangian function over the hypothesis class is called the primal minimizer,
\begin{equation}\label{eq:primal-min-stat}
f_{\theta}(\lambda) \in \arg \min_{f_\theta \in \mathcal{H}} {L}(f_\theta,\lambda),
\end{equation}
and its value is known as the dual function $g: L^1_+(\mathcal{D}) \to \reals$,
\begin{equation}
    g(\lambda) := L(f_{\theta}(\lambda), \lambda) = \min_{f_{\theta} \in \mathcal{H}} L(f_{\theta}, \lambda).
\end{equation}
The dual problem maximizes the dual function over all admissible nonnegative multipliers. When the supremum is attained, we denote by \(\lambda^\star\) an optimal multiplier, so that
\begin{equation}\label{P:dual_statistical}
\tag{D}
	D^\star
	:=
	\sup_{\lambda \in L^1_+(\mathcal{D})} g(\lambda)
	=
	g(\lambda^\star).
\end{equation}

In convex settings, mild conditions guarantee that the dual and primal problems are equivalent. In our setting, standard parametrization choices for \(f_\theta\) (e.g., neural networks) render the learning problem~\eqref{P:primal_statistical} not convex. In this section, we focus on the learnability of the dual problem itself, whereas in Section~\ref{sec:dual-imply-pac} we show that rich parametrization classes \(\mathcal H\) can control the gap induced by working with the dual problem.

\subsection{Dual learnability}

We provide a PAC-style definition for learnability of the dual problem, as Definition~\ref{def:pac} does not handle the hierarchical structure of a problem with a nested optimization.
\begin{definition}[Dual PAC-learnability]\label{def:dual-pac}
A hypothesis class \(\mathcal{H}\) is \emph{probably approximately correct} dual learnable if, for every \(\epsilon,\delta \in (0,1)\) and every distribution \(\mathcal D\), there exists a learning rule such that, given \(N \ge N_{\mathcal H}(\epsilon,\delta)\) samples, it returns estimators \(\hat \lambda \in L^1_+(\mathcal{D})\) and \(\hat f_\theta \in \mathcal{H}\) satisfying, with probability at least \(1-\delta\),
\begin{equation}\label{eq:dual_pac_obj}
\big| D^* - L(\hat f_\theta, \hat \lambda) \big| \le \epsilon.
\end{equation}
\end{definition}

Definition \ref{def:dual-pac} parallels the standard PAC definition, as it requires the statistical Lagrangian evaluated at the returned primal-dual pair to be close to the optimal value $D^*$. The definition is motivated by related notions for saddle-point and minimax learning problems~\cite{zhang2020generalizationboundsstochasticsaddle,ozdaglar2022goodmetricstudygeneralization,farnia2020trainsimultaneouslygeneralizebetter}, but adapted to the dual formulation considered here. 

We formulate the empirical dual rule by taking the dual of the empirical primal problem~\eqref{P:primal_empirical}. This corresponds to replacing the expectation in~\eqref{E:lagrangian_csl_param} with a sample average and assigning a nonnegative multiplier to each training sample. Explicitly, let \(S_0=\{\bz_n\}_{n=1}^{N}\) and \(S=\{\bz_n\}_{n=1}^{N}\) be collections of  i.i.d.\ samples drawn from \(\mathcal{D}_0\) and \(\mathcal{D}\) respectively, and define the empirical Lagrangian \(\hat L(f_\theta,\lambda):\mathcal{H}\times\mathbb{R}_+^N \to \mathbb{R}\) by
\begin{equation}
\label{E:empirical_lagrangian}
\hat{L}(f_\theta, \lambda) =
\frac{1}{N_0} \sum_{\bz_n \in S_0} \ell_0\big( f_{\theta},\bz_n \big)
+ \frac{1}{N} \sum_{\bz_n \in S} \lambda_n \ell\big( f_{\theta},\bz_n \big).
\end{equation}
As in the statistical case, we define the empirical primal minimizer,
\begin{equation}\label{eq:primal-min-stat}
\hat f_{\theta}(\lambda) \in \arg \min_{f_\theta \in \mathcal{H}} \hat L (f_\theta,\lambda),
\end{equation}
whose value yields the empirical dual function
\begin{equation}
    \hat g(\lambda) := \hat  L(\hat  f_{\theta}(\lambda), \lambda) = \min_{f_{\theta} \in \mathcal{H}} \hat L(f_{\theta}, \lambda).
\end{equation}
Maximizing over the set of admissible multipliers results in the empirical dual problem. When a maximizer exists, we denote it by \(\hat\lambda^\star\), so that
\begin{equation}\label{P:dual_empirical}
\tag{$\hat{\textup{D}}$}
	\hat{D}^\star = \max_{\lambda \in \reals_+^N} \hat g(\lambda) = \hat g(\hat\lambda^\star).
\end{equation}

We assume the existence of optimal multipliers for both the statistical and empirical dual problems. We state this condition explicitly so that it can be invoked throughout the following results.
\begin{assumption} \label{ass:dual_exis}
    The statistical and empirical dual optimal solutions \(\lambda^\star\) and
    \(\hat\lambda^\star\) exist.
\end{assumption}
Assumption~\ref{ass:dual_exis} is imposed here to avoid technical measure-theoretic arguments. Appendix~\ref{apx:distr_space} in the supplementary material  shows that by working with dual variables that live in a measure space, this assumption can be replaced with a strict feasibility condition.

While Definition~\ref{def:dual-pac} gives a general notion of dual learnability, the empirical rule~\eqref{P:dual_empirical} admits a simpler equivalent criterion.
\begin{lemma}[Optimal-value convergence implies dual learnability]\label{prop:dual_equiv}
Let Assumptions~\ref{ass:unif_conv} and~\ref{ass:dual_exis} hold. Suppose that for every \(\epsilon,\delta \in (0,1)\) and every distribution \(\mathcal D\), whenever \(N \ge N_{\mathcal H}(\epsilon,\delta)\), it holds with probability at least \(1-\delta\) that
\begin{equation}
\big| D^* - \hat D^* \big| \le \epsilon .
\end{equation}    
Then~\eqref{P:dual_statistical} is dual PAC learnable (Definition~\ref{def:dual-pac}).
\end{lemma}
\begin{proof}
    See Appendix~\ref{apx:dual_equiv} in the supplementary material. 
\end{proof}

Lemma~\ref{prop:dual_equiv} shows that, for the empirical rule~\eqref{P:dual_empirical}, dual learnability can be certified by bounding the difference between the statistical and empirical optimal dual values. We use this characterization next to derive sufficient conditions for dual learnability.
\subsection{Sufficient condition for dual learnability}\label{sec:suff-cond}
To prove dual learnability, we require that the optimal dual multiplier \(\lambda^*\) has a bounded \(L_\infty\) norm.
\begin{assumption}\label{ass:bounded-inf-norm}
The optimal dual multiplier \(\lambda^*\) satisfies
\begin{equation}\label{eq:ass-prop-6}
\|\lambda^*\|_\infty \leq \gamma < \infty .
\end{equation}
\end{assumption}
Assumption~\ref{ass:bounded-inf-norm} is natural from the sensitivity interpretation of dual variables. In Subsection~\ref{sec:ass-disc} we show that the multiplier \(\lambda^*(\bz)\) measures the impact that relaxing a constraint at sample \(\bz\) has on the optimal dual value relative to the probability of the sample. Therefore, the assumption prevents the dual problem from being arbitrarily sensitive to low-probability samples. Although natural, the assumption is not necessarily mild. This is why in Subsection~\ref{subsec:sparse} we show that, even when the assumption does not hold for the original problem, it can be enforced through a sparse relaxation of the constraints.

We now state the resulting learnability guarantee.

\begin{theorem}[Sufficient condition for dual learnability]\label{thm:dual-learn}
Suppose Assumptions~\ref{ass:unif_conv}--\ref{ass:bounded-inf-norm} hold. Then, with probability at least \(1-3\delta\),
\begin{equation}\label{eq:dual-lear-rate}
    |D^* - \hat D^*|
    \leq
    2\zeta_0(N,\delta) + \gamma\,\zeta(N,\delta),
\end{equation}
and, consequently, \eqref{P:dual_statistical} is dual learnable.
\end{theorem}
\begin{proof} We provide the proof of Theorem~\ref{thm:dual-learn} as the combination of two lemmas.

\begin{lemma}[Uniform convergence of the Lagrangian]\label{lemma:suff-cond-lags}
Let Assumptions~\ref{ass:unif_conv} and~\ref{ass:dual_exis} hold. Suppose there exists a sample-dependent multiplier \(\tilde \lambda \in \mathbb{R}_+^N\) and a function \(\zeta_L(N,\delta)\), monotonically decreasing in \(N\), such that, with probability at least \(1-\delta\),
\begin{equation}\label{eq:unif-conv-lagrangians}
    |L(f_\theta,\lambda^*) - \hat L(f_\theta, \tilde \lambda)|
    \leq
    \zeta_L(N,\delta),
\end{equation}
uniformly over \(f_\theta \in \mathcal{H}\). Then the estimator
\[
\hat f_\theta^* \in \arg \min_{f_\theta \in \mathcal{H}}
\hat L(f_\theta,\hat \lambda^*)
\]
returned by rule~\eqref{P:dual_empirical} satisfies, with probability at least \(1-2\delta\),
\begin{equation}
    |D^* -  \hat D^*|
    \leq
    \zeta_0(N,\delta) + \zeta_L(N,\delta).
\end{equation}
Consequently, \eqref{P:dual_statistical} is dual learnable.
\end{lemma}

The condition in~\eqref{eq:unif-conv-lagrangians} parallels the uniform convergence condition~\eqref{eq:unif-conv-obj} used in standard PAC learning. It requires the statistical and empirical Lagrangians to become arbitrarily close as the number of samples increases, uniformly over the hypothesis class \(\mathcal H\). We refer to~\eqref{eq:unif-conv-lagrangians} as the uniform convergence of the Lagrangian property. Importantly, this condition only needs to hold for the statistical Lagrangian evaluated at the optimal multiplier \(\lambda^*\), rather than uniformly over all multipliers \(\lambda\). Thus, Lemma~\ref{lemma:suff-cond-lags} simplifies the characterization of dual learnability in the same way that Lemma~\ref{prop:unif-conv-imply-pac} does for standard PAC learning. The next lemma shows that boundedness of \(\lambda^*\) is sufficient for~\eqref{eq:unif-conv-lagrangians} to hold.

\begin{lemma}[Bounded multipliers imply Lagrangian uniform convergence]
\label{lemma:dual-learn}
Suppose Assumptions~\ref{ass:unif_conv}-\ref{ass:bounded-inf-norm} hold. Let the sample-dependent multiplier \(\tilde{\lambda}\in\mathbb{R}_+^N\) be defined by
\begin{equation}
    \tilde{\lambda}_n = \lambda^*(\bz_n),
    \qquad n=1,\ldots,N .
\end{equation}
Then, with probability at least \(1-2\delta\),
\begin{equation}
    \big|L(f_\theta,\lambda^*) - \hat L(f_\theta,\tilde{\lambda})\big|
    \leq
    \zeta_0(N,\delta) + \gamma\,\zeta(N,\delta)
\end{equation}
uniformly over \(f_\theta \in \mathcal{H}\).
\end{lemma}

Lemma~\ref{lemma:dual-learn} explains that the rate in~\eqref{eq:dual-lear-rate} comes from comparing the statistical Lagrangian evaluated at \(\lambda^*\) with the empirical Lagrangian evaluated at the sampled multiplier \(\tilde\lambda_n=\lambda^*(\bz_n)\). In particular, the lemma verifies the uniform convergence condition required by Lemma~\ref{lemma:suff-cond-lags}, and therefore immediately yields Theorem~\ref{thm:dual-learn}.

The proof of Lemma~\ref{lemma:suff-cond-lags} and the proof of Lemma~\ref{lemma:dual-learn} are in Appendix~\ref{apx:suff-cond} and Appendix~\ref{apx:dual_learning_2} in the supplementary material correspondingly. 
\end{proof}

Theorem~\ref{thm:dual-learn} shows that \eqref{P:dual_statistical} is learnable whenever the optimal multiplier has bounded \(L_\infty\) norm. The resulting rate is governed by the slower of the objective and constraint uniform-convergence rates, with the constraint term scaled by the multiplier bound \(\gamma\). Thus, under this boundedness condition, dual learnability is no harder than standard PAC learnability from a statistical standpoint. 

\subsection{Sparse Constraint Relaxation}\label{subsec:sparse}

\begin{figure*}[t]
\centering
    \includegraphics[width=0.9\linewidth]{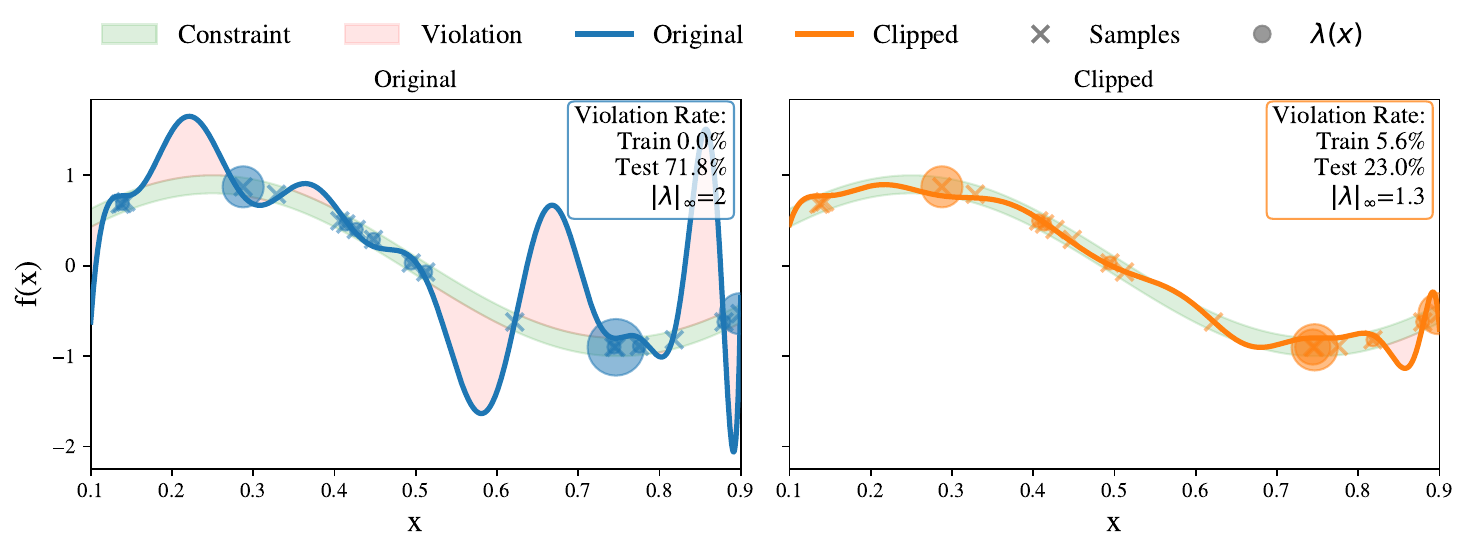}
\caption{Total-variation polynomial fits to samples generated by adding uniform noise to a sine wave. The original problem minimizes the variation of the fitted polynomial subject to the pointwise constraints \(\lvert f_\theta(x_i)-y_i\rvert \leq \epsilon\) at every training sample. The clipped problem restricts the dual variables to \(\lVert\lambda\rVert_\infty \leq 1.3\), equivalently allowing nonnegative per-sample constraint slacks \(u_i\) penalized by \(1.3\sum_i u_i\). Circle sizes are proportional to the learned dual variables. 
\emph{Clipping the dual variables relaxes the constraints at a small fraction of the training samples, recovering a smoother solution that generalizes better.}}\label{fig:polyreg}
\end{figure*}

Subsection~\ref{sec:suff-cond} concluded that a sufficient condition for dual learnability is boundedness of the optimal multiplier $\lambda^*$ (Assumption~\ref{ass:bounded-inf-norm}). As this may or may not hold in practice, we consider how to enforce it. A direct approach is to restrict the dual domain. To this end, we define the set of bounded multipliers
\begin{equation}
\Lambda_\gamma
=
\left\{
\lambda \in L^1_+(\mathcal{D}) :
\|\lambda\|_\infty \leq \gamma < \infty
\right\}.
\end{equation}
We then consider the restricted dual problem, where by construction the solution satisfies the boundedness condition in Assumption~\ref{ass:bounded-inf-norm}.
\begin{equation}\tag{$\textup{D}_\gamma$}\label{P:dual_res}
    D_{\gamma}^*
    =
    \sup_{\lambda \in \Lambda_\gamma} g(\lambda).
\end{equation}

An alternative approach comes from the sensitivity interpretation of dual multipliers. Large values of the optimal multiplier indicate that the problem is highly sensitive to the corresponding constraint. Thus, a natural way to reduce large multipliers is to relax such constraints. Ideally, this relaxation should be minimal, modifying the original constrained problem only where necessary. This motivates the following relaxed primal formulation:
\begin{equation}\label{P:primal_res}\tag{$\textup{P}_\gamma$}
\begin{aligned}
	P^\star_{\gamma}
	=
	\min_{f_\theta \in \mathcal{H},\, u\in L_1(\mathcal D)}
	\quad
	&
	\E_{\bz \sim \mathcal D_0}
	\!\Big[
	    \ell_0\big( f_{\theta}, \bz \big)
	\Big]
        + \gamma \E_{\bz\sim\mathcal D}\!\Big[|u(\bz)|\Big]
	\\
	\text{s.t.}
	\quad
	&
	\ell\big( f_{\theta}, \bz \big)
	\leq
	c + u(\bz)
    \quad \mathcal D\text{-a.e.}
\end{aligned}
\end{equation}
Problem~\eqref{P:primal_res} relaxes~\eqref{P:primal_statistical} by introducing the slack variable \(u(\bz)\), which permits pointwise violations of the original constraints. The \(L_1\) penalty discourages relaxations. In particular, because \(L_1\) penalties promote sparsity, the formulation favors relaxing the constraints only on a limited subset of samples; see, e.g.,~\cite{ramirez2013l1,chamon2020functional}.

Although the restricted dual problem~\eqref{P:dual_res} and the relaxed primal problem~\eqref{P:primal_res} were introduced from different perspectives, the next proposition shows that they are tightly linked.
\begin{proposition}\label{prop:duality_res}
   Problem~\eqref{P:dual_res} is the dual of problem~\eqref{P:primal_res}.
\end{proposition}
\begin{proof}
    See Appendix~\ref{apx:dual_res} in the supplementary material. 
\end{proof}
Proposition~\ref{prop:duality_res} provides an interpretation of constraining the domain of the dual multipliers to the set $\Lambda_\gamma$, which corresponds to a sparse relaxation of the constraints in the primal domain.

Since every solution of~\eqref{P:dual_res} has bounded multipliers by construction, the learnability analysis can be extended naturally to this restricted problem. The empirical counterpart must enforce the same \(L_\infty\) bound on the multipliers, leading to the following result.
\begin{theorem}\label{thm:dual_learnability_res}
Let Assumption~\ref{ass:unif_conv} hold, and define
\begin{equation}\label{P:dual_res_emp}\tag{$\hat{\textup{D}}_\gamma$}
    \hat{D}_{\gamma}^\star = \max_{\lambda \in [0,\gamma]^N}\ \hat g(\lambda).
\end{equation}
Then the solutions \(\lambda^*_{\gamma}\) and \(\hat \lambda^*_{\gamma}\) of~\eqref{P:dual_res} and~\eqref{P:dual_res_emp} exist, and the estimator
\[
\hat f_\theta^* \in \arg \min_{f_\theta \in \mathcal{H}} \hat L(f_\theta, \hat \lambda^*_{\gamma})
\]
satisfies, with probability at least \(1-5\delta\),
\begin{equation}
|D_{\gamma}^* - \hat L(\hat f_\theta^*, \hat \lambda^*_{\gamma})|
\leq
2\zeta_0(N,\delta) + \gamma \zeta(N,\delta).
\end{equation}
\end{theorem}
\begin{proof}
    See Appendix~\ref{apx:res_learnability} in the supplementary material. 
\end{proof}
Theorem~\ref{thm:dual_learnability_res} shows that, for the restricted dual problem, uniform convergence of the objective and constraint losses is sufficient to establish dual learnability. This introduces a trade-off controlled by \(\gamma\): smaller values improve the statistical bound, but also lower the cost of the slack variable, allowing larger relaxations in the primal problem. Figure~\ref{fig:polyreg} illustrates how introducing the relaxation can be beneficial.  Clipping the dual variables prevents overly restrictive constraints and leads to smoother polynomial fits with better generalization. Moreover, from an implementation standpoint, the relaxation is especially simple. There is no need to introduce the slack variable \(u\) explicitly; it is only necessary to enforce the bound \(\lambda_n \leq \gamma\) through clipping.

\subsection{Discussion}
\label{sec:ass-disc}

Earlier in this section, we argued that bounding \(\|\lambda^*\|_\infty\) is a natural condition for dual learnability. This subsection justifies that claim.  Appendix~\ref{apx:distr_space} in the supplementary material 
formalizes the arguments presented here through the lens of measure theory.

\vspace{0.2cm}
\noindent
\textbf{Multipliers induce a reweighting measure.} Let  \(P_\lambda\) be the nonnegative measure induced by the multiplier $\lambda$, defined by
\begin{equation}\label{eq:dist-lambda-def}
    P_\lambda(\bz)
    :=
    \lambda(\bz)\, \mathcal D(\bz).
\end{equation}
Under this representation, the Lagrangian in~\eqref{E:lagrangian_csl_param} can be written equivalently as
\begin{equation}\label{E:lagrangian_csl_measure}
	L(f_{\theta}, P_\lambda)
    :=
    \E_{\bz \sim \mathcal D_0}
    \!\Big[
        \ell_0\big( f_{\theta},\bz \big)
    \Big]
    +
    \E_{\bz \sim P_\lambda}\Big[
        \ell(f_\theta,\bz)\Big].
\end{equation}
The measure \(P_\lambda\) can be interpreted as a reweighting of the sample space \(\mathcal Z\) in the constraint term of~\eqref{E:lagrangian_csl_measure}. For a fixed measure \(P_\lambda\), the Lagrangian function in \eqref{E:lagrangian_csl_measure} is independent of \(\mathcal D\). Thus, \(P_\lambda\) assigns penalty mass to constraint violations without accounting for the probability of the region where they occur.

\vspace{0.2cm}
\noindent
\textbf{The reweighting measure quantifies sensitivity.}
Optimal dual multipliers are known to quantify the sensitivity of an optimization problem to its constraints~\cite[Chapter 5.6]{boyd2004convex}. To make this interpretation precise, let \(\bu \in L_1(\mathcal D)\) be a function that assigns to each data point \(\bz\) a perturbation \(\bu(\bz)\) of the constraint \(\ell(f_\theta,\bz)\). We define the perturbation function \(P(\bu)\) as the optimal value of problem~\eqref{P:primal_statistical} after perturbing the constraints by \(\bu\):
\begin{equation}\label{P:primal_perturbed}\tag{$\text{P}_\text{u}$}
\begin{aligned}
	P(\bu) = \min_{f_\theta \in \mathcal{H}}&
		&&\E_{\bz \sim \mathcal D_0} \!\Big[ \ell_0\big( f_{\theta},\bz \big) \Big]
	\\
	\text{s. to}& &&  \ell \big( f_{\theta},\bz \big)  \leq \bu(\bz) \quad \mathcal D\text{-a.e.}
\end{aligned}
\end{equation}
For each fixed perturbation \(\bu\), problem~\eqref{P:primal_perturbed} is a constrained optimization problem and therefore has an associated dual problem. We define the dual perturbation function \(D(\bu)\) as the optimal value of this dual problem. When \(\bu=0\), we recover the original constrained problem. The following lemma formalizes the sense in which the optimal multiplier measures the sensitivity of the optimal value to perturbations of the constraints.

\begin{lemma}\label{lemma:dual-are-subg}
The optimal dual measure \(P_\lambda^*\) is a subgradient of the dual perturbation function \(D(\bu)\) at \(\bu=0\). Equivalently, for every perturbation \(\bu\),
\begin{equation}
    D(\bu) \geq D(0) - \E_{\bz \sim P_\lambda^*}\Big[u(\bz)\Big].
\end{equation}
\end{lemma}

Since \(D(\bu)\) is monotonically decreasing in \(\bu\), any relaxation of the constraints, corresponding to positive \(\bu\), decreases the value with respect to \(D(0)\). Lemma~\ref{lemma:dual-are-subg} bounds the magnitude of this decrease. The larger the value of \(P_\lambda^*(\bz)\) on a sample or region, the larger the effect of relaxing the corresponding constraint. In the limiting case where \(P_\lambda^*(\bz)=0\), relaxing the constraint at that sample has no effect on the optimal value.

\vspace{0.2cm}
\noindent
\textbf{The empirical problem is a relaxation of \eqref{P:primal_statistical}.}
The statistical constraints are required to hold \(\mathcal D\)-almost everywhere, whereas the empirical problem enforces them only on the sampled points. Thus, passing from the statistical problem to the empirical problem can be viewed as relaxing the constraints on regions that are not sampled. From this perspective, \(P_\lambda^*\) measures the sensitivity of the problem to each region in the space, since it controls how much the optimal value can change when that region is not represented in the sample set.

\begin{remark}
    The optimal measure \(P_\lambda^*\) quantifies the sensitivity of the problem to each region of the space, because it bounds the effect that not sampling that region can have on the optimal value.
\end{remark}

\vspace{0.2cm}
\noindent
\textbf{Learnability depends on the mismatch between $P_\lambda^*$ and $\mathcal{D}$.} 
Since the empirical problem only enforces constraints on sampled points, the effect of missing a region depends on two quantities: the sensitivity of the optimum to relaxing the constraints on that region, and the probability that the region is sampled. The first quantity is captured by \(P_\lambda^*\), while the second is determined by \(\mathcal D\). Therefore, the difficulty of learning the dual problem is naturally governed by the mismatch between them. This mismatch is precisely what the functional multiplier \(\lambda^*\) measures. Applying~\eqref{eq:dist-lambda-def} to the optimal dual solution gives
\begin{equation}\label{eq:dist-lambda-optimum}
    \lambda^*(\bz)
    :=
    \frac{P_\lambda^*(\bz)}{\mathcal D(\bz)}.
\end{equation}
Thus, \(\lambda^*(\bz)\) measures the sensitivity assigned to a data point \(\bz\) relative to the probability of sampling it. Large values of \(\lambda^*(\bz)\) indicate regions where \(P_\lambda^*\) assigns much more mass than \(\mathcal D\). These are regions that have high impact on the optimum but are unlikely to be observed in the empirical sample. 

\begin{remark}
    Bounding \(\|\lambda^*\|_\infty\) limits the sensitivity that any region can have relative to its probability of being sampled.
\end{remark}

Figure~\ref{fig:distr-learning} illustrates this effect by comparing cases where \(P_\lambda^*\) and \(\mathcal D\) are well aligned or poorly aligned.

\begin{figure*}[t]
    \centering
    \includegraphics[width=0.9\linewidth]{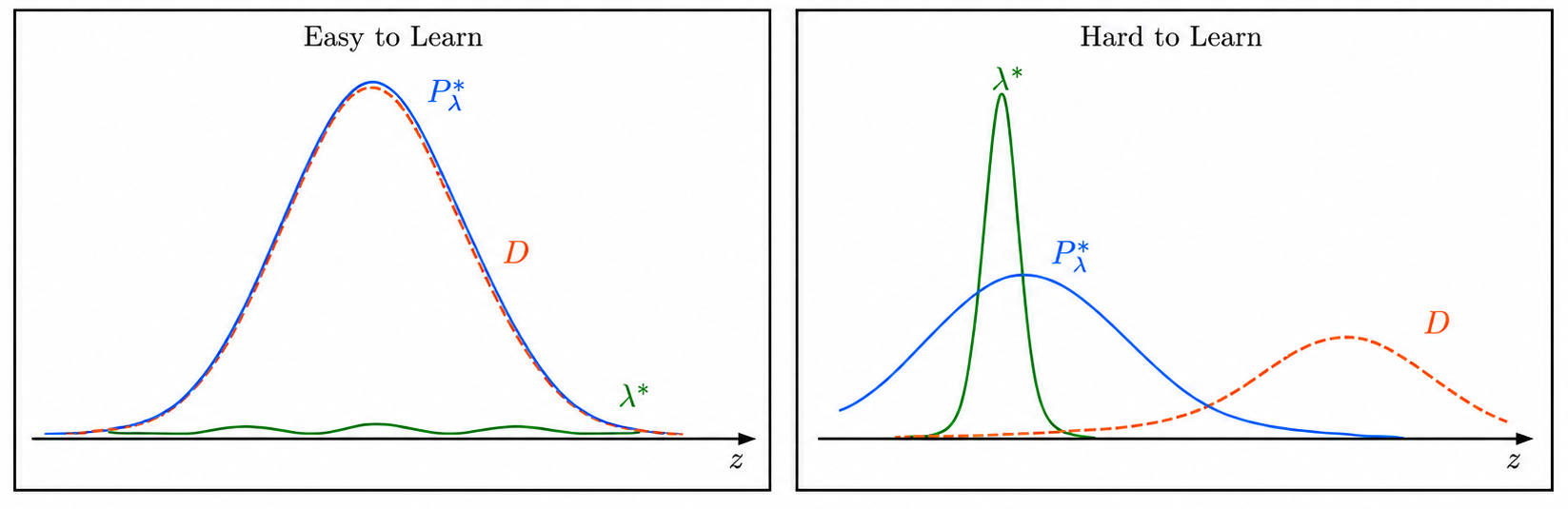}
    
    \caption{
    Illustration of the relationship between the overlap of \(P_\lambda^\star\) and \(\mathcal D\), and the magnitude of the induced optimal multiplier \(\lambda^\star\).
    When \(P_\lambda^\star\) is well aligned with the data distribution \(\mathcal D\) (left), the corresponding multiplier remains small across the domain, indicating that no region requires a large sensitivity weight.
    In contrast, when \(P_\lambda^\star\) and \(\mathcal D\) are poorly aligned (right), the multiplier concentrates around regions where data probability is low with respect to the sensitivity mass, leading to large values of \(\lambda^\star\).
    This illustrates that poor overlap can induce large optimal multipliers, making the problem harder to learn.
    }
    \label{fig:distr-learning}
\end{figure*}

\vspace{0.2cm}
\noindent
\textbf{Dual learnability is not equivalent to uniform convergence.}
We close the discussion with an example in which \(P_\lambda^*\) concentrates all its mass at a single point, so that the corresponding functional multiplier \(\lambda^*\) goes unbounded. This shows that, without additional control on \(\lambda^*\), uniform convergence of the objective and constraint losses alone can not provide a dual learnability convergence rate.

\begin{example}\label{ex:counter-uc-lambda}
Let \(\theta \in \mathbb R\) be the decision variable, and let \(x \in [0,1]\) be a random variable with cumulative distribution function \(F\). Consider the statistical constrained problem
\begin{equation}\label{eq:counterexample_inf_norm}
\begin{aligned}
    & \min_{\theta \in \mathbb R} \quad  (\theta-1)^2 \\
    &\text{s.t.} \quad \theta - x \leq 0, \quad \mathcal{D}\text{-a.e.}.
\end{aligned}
\end{equation}
This problem can be solved by inspection. Since the constraint requires \(\theta \leq x\) \(\mathcal D\)-almost everywhere, the binding requirement is imposed by the smallest value of \(x\), namely \(x=0\). Hence, the feasible set reduces to \(\theta \leq 0\), and the minimizer of \((\theta-1)^2\) over this set is \(\theta^\star=0\), with optimal value \(P^\star=1\).

Although the primal solution is immediate by inspection, solving the example through the dual problem makes the role of the multiplier explicit. The dual problem of~\eqref{eq:counterexample_inf_norm} is
\begin{equation}\label{eq:counterexample_dual}
    g(P_\lambda)
    =
    - \E_{P_\lambda} \big[1\big]
    \left(
        \frac{\E_{P_\lambda} \big[1\big]}{4}
        -1
    \right)
    -
    \E_{P_\lambda} \big[x\big].
\end{equation}
The solution of~\eqref{eq:counterexample_dual} is \(P_\lambda^* = 2\delta_0\),\footnote{\(\delta_0\) denotes a Dirac delta at \(0\).} which yields the value \(D^*=1\).

To construct the empirical problem, let \(S=\{x_i\}_{i=1}^N\) be a set of \(N\) i.i.d.\ samples from \(F\), ordered increasingly, so that \(x_1\) is the smallest sample. The optimal empirical value is \(\hat D^* = (1-x_1)^2\). Hence, the learnability gap is
\begin{equation}
    |D^* - \hat D^*|
    =
    1-(1-x_1)^2
    =
    2x_1-x_1^2
    \leq
    2x_1 .
\end{equation}

Given the CDF \(F(x)\), with probability \(1-\delta\) over the sample set, the smallest sample satisfies
\begin{equation}\label{eq:x1_rate}
    x_1
    \leq
    F^{-1}(1-\delta^{1/N})
    \approx
    F^{-1}\Big(\frac{\log (1/\delta)}{N}\Big).
\end{equation}

If \(F\) grows linearly near \(0\), as in the uniform distribution, then \(F^{-1}\) is also linear and the rate is of order \(1/N\). By contrast, if \(F(x)\approx e^{-c/x}\) near \(0\), then \(F^{-1}(u)\approx c/\log(1/u)\), yielding a rate of order \(1/\log N\). Therefore, when Assumption~\ref{ass:bounded-inf-norm} fails, the learning rate can be arbitrarily slow, even in this simple one-dimensional example.

\end{example}

Example~\ref{ex:counter-uc-lambda} shows that, even when uniform convergence holds for both the objective and constraint losses, dual learnability can fail in the sense that no distribution-free convergence rate can be guaranteed. Indeed, the objective loss \(\ell_0(\theta,x)=(\theta-1)^2\) does not depend on \(x\), so its empirical and statistical averages coincide. Moreover, for the affine constraint loss \(\ell(\theta,x)=\theta-x\), the empirical-statistical convergence rate reduces to the standard \(N^{-1/2}\) rate from the law of large numbers. Therefore, the usual uniform convergence condition holds. However, since \(P_\lambda^*=2\delta_0\), all dual mass is concentrated at the single point \(x=0\). Thus, the corresponding functional multiplier \(\lambda^*\) is unbounded. As a result, the dual approximation error depends on how close the empirical sample set gets to \(0\), which is quantified by the inverse distribution function \(\mathcal{F}^{-1}\) near \(0\).

This conclusion is consistent with the solution obtained by inspection. Since feasibility is determined by the smallest value \(x=0\), it is natural that the optimal dual measure assigns all its mass to that point. Consequently, when the problem is solved empirically, the quality of the solution depends on how easily the sample set approaches \(x=0\), and this is precisely what the learnability rate in~\eqref{eq:x1_rate} captures.

\section{Primal Near Learnability }
\label{sec:dual-imply-pac}

Section~\ref{sec:dual} established sufficient conditions for learnability of the dual problem. In this section, we relate dual learnability back to primal learnability. We first show that, when~\eqref{P:primal_statistical} is convex, learnability of the dual problem directly implies PACC learnability of the primal problem; in this case, both the empirical dual rule~\eqref{P:dual_empirical} and the constrained empirical rule~\eqref{P:primal_empirical} can be used as PACC learners. Modern parametrizations, such as neural networks and transformers, are typically nonconvex, and therefore~\eqref{P:primal_statistical} is generally nonconvex as well. For this setting, we show that the learning rule~\eqref{P:primal_empirical} remains PACC feasible and near PACC optimal, up to an approximation term that depends on the richness of the parametrization class.

\subsection{Learnability for convex problems}

A constrained optimization problem is strongly convex when its objective function is strongly convex and its feasible set is convex. In~\eqref{P:primal_statistical}, convexity of the feasible set follows from convexity of the hypothesis class \(\mathcal H\) together with convexity of the constraint loss.

\begin{assumption}\label{ass:convx-H}[Convexity of the parametric class]
    The hypothesis set $\mathcal{H}$ is convex.
\end{assumption}

\begin{assumption}\label{ass:strong-convx-loss}[Strong convexity of the losses]
    The loss functions $\ell$ and $\ell_0$ are convex and strongly convex, respectively.
\end{assumption}

Assumption~\ref{ass:strong-convx-loss} is standard in many learning problems, holding for losses such as the squared loss or negative log-likelihood. Assumption~\ref{ass:convx-H}, however, is satisfied only by classical parametrizations, such as linear predictors or kernel methods. Modern parametrizations, including neural networks and transformers, typically violate this convexity assumption. In Subsection~\ref{subsec:primal-nonconvex}, we replace Assumption~\ref{ass:convx-H} with a universality condition on \(\mathcal H\), which is more appropriate for such rich parametrization classes.

In constrained problems, convexity together with dual attainability (Assumption~\ref{ass:dual_exis}) implies strong duality, so the primal and dual optimal values are the same. As a result, dual learnability guarantees can be transferred directly to PACC learnability for the empirical primal rule~\eqref{P:primal_empirical}.

\begin{theorem}[PACC learnability for convex \(\mathcal{H}\)]
\label{thm:pacc}
Let Assumptions~\ref{ass:unif_conv}--\ref{ass:strong-convx-loss} hold, and let \(\bar f_\theta\) be the estimator returned by \eqref{P:primal_empirical}. Then, with probability at least \(1-6\delta\),
\begin{align}
\left|
\mathbb{E}_{\bz \sim \mathcal D_0}
\big[
\ell_0(\bar f_\theta,\bz)
\big]
-
P^\star
\right|
&\le
 3\zeta_0 + \gamma \zeta,
\label{eq:pacc_obj_thm}
\\
\Pr_{\bz\sim\mathcal D}\big[\ell(\bar f_\theta,\bz) \le 0\big]
&\geq
1-\zeta_I.
\label{eq:pacc_const_thm}
\end{align}
Therefore, Problem~\eqref{P:primal_statistical} is PACC-learnable.
\end{theorem}
\begin{proof}
    See Appendix~\ref{apx:pacc_convex} in the supplementary material. 
\end{proof}

Theorem~\ref{thm:pacc} shows that, in the convex setting, solving the empirical primal problem yields an estimator that is PACC learnable. The optimality guarantee follows from dual learnability and strong duality. Consequently, its rate matches the dual learnability rate in Theorem~\ref{thm:dual-learn}, up to an additional \(\zeta_0\) term. The feasibility guarantee follows because the estimator in~\eqref{eq:estimator_pacc} is feasible for the empirical constrained problem~\eqref{P:primal_empirical}; uniform convergence of the constraint indicator \(\mathbb{I}\{\ell\leq 0\}\) then transfers empirical feasibility to statistical feasibility, with rate \(\zeta_I\).

In many applications, however, projection onto the feasible set may be prohibitively costly, and thus solving~\eqref{P:primal_empirical} directly is impractical. This motivates solving the empirical dual problem~\eqref{P:dual_empirical} instead. The following proposition shows that, under the convexity assumptions used above, this dual rule returns the same estimator as the primal rule.

\begin{proposition}\label{prop:primal-est-same-dual}
Let Assumptions~\ref{ass:unif_conv}--\ref{ass:strong-convx-loss} hold, and let $\bar f_\theta$ be the estimator returned by~\eqref{P:dual_empirical},
\begin{equation}\label{eq:estimator_pacc}
    \bar f_\theta
    =
    \arg \min_{f_\theta \in \mathcal{H}}
    \hat L(f_\theta, \hat \lambda^*).
\end{equation}
The estimator \(\bar f_\theta\) is the same as the estimator returned by~\eqref{P:primal_empirical}.
\end{proposition}
\begin{proof}
    See Appendix~\ref{apx:pacc_dual} in the supplementary material. 
\end{proof}

Proposition~\ref{prop:primal-est-same-dual} implies that the dual rule~\eqref{P:dual_empirical} is also a PACC learner. Indeed, since the empirical primal and dual rules return the same predictor, the guarantees of Theorem~\ref{thm:pacc} apply directly to the estimator returned by \eqref{P:dual_empirical} and thus both formulations yields the same statistical guarantees.

\subsection{Near learnability for modern AI problems}\label{subsec:primal-nonconvex}

For modern machine learning parameterizations, such as neural networks, Assumption~\ref{ass:convx-H} does not hold. We show that primal learnability can still be established up to the duality gap of the problem, which is controlled by the richness of the parametrization. To make this precise, let
\begin{equation}
\hat{\mathcal{H}} = \overline{\mathrm{conv}}(\mathcal{H})
\end{equation}
be the closed convex hull of the hypothesis class.\footnote{The closure is taken with respect to the \(\|\cdot\|_\infty\) norm, so that uniform limits preserve continuity and boundedness.} We consider the functional problem resulting from replacing in \eqref{P:primal_statistical} the domain $\mathcal H$ for its closed convex hull $\hat{\mathcal H}$.
\begin{equation}
\label{P:csl_variational}\tag{$\tilde{\textup{P}}$}
\begin{aligned}
\widetilde{P}^\star = \min_{\phi \in \hat{\mathcal{H}}}&\;
   \E_{\bz \sim \mathcal D_0} \!\Big[ \ell_0\big( \phi,\bz \big) \Big] \\
\text{s.t.}\quad &  \ell \big( \phi, \bz\big) 
    \leq 0 \quad \mathcal{D}\text{-a.e.}
\end{aligned}
\end{equation}



Under Assumption~\ref{ass:strong-convx-loss}, problem~\eqref{P:csl_variational} is a convex relaxation of problem~\eqref{P:primal_statistical}. Together with a dual attainability assumption, this yields a convex relaxation for which strong duality holds. If, in addition, the hypothesis class \(\mathcal H\) is close to the convexified class \(\hat{\mathcal H}\), then the original problem~\eqref{P:primal_statistical} and its convex relaxation~\eqref{P:csl_variational} should have similar optimal values. Thus, the duality gap of~\eqref{P:primal_statistical} can be bounded in terms of the approximation error between \(\mathcal H\) and \(\hat{\mathcal H}\).

\begin{assumption}[Near universality]\label{ass:near-univ}
There exist constants \(\nu,\nu_0 \geq 0\) such that, for every \(\phi \in \hat{\mathcal{H}}\), there exists a \(f_\theta \in \mathcal{H}\) satisfying
\begin{align} 
\mathbb{E}_{\mathcal{D}_0}\left[\left|\phi(\bx)-f_{\theta}(\bx)\right|\right] \leq \nu_0 ,
\\
|\phi(\bx)-f_{\theta}(\bx)| \leq \nu \quad \mathcal{D}\text{-a.e.}
\label{eqn_pointwise_near_univ}
\end{align}
\end{assumption}

\begin{assumption}[Lipschitz losses]\label{ass:lip_out}
The loss functions \(\ell(\hat \by, \by)\) and \(\ell_0(\hat \by, \by)\) are \(M\)-Lipschitz continuous in their first argument\footnote{For this assumption, we refer to the original definition of the losses \(\ell_i:\mathcal{Y}\times\mathcal{Y}\) rather than to the reparameterized form \(\ell_i:\mathcal{H}\times\mathcal{Z}\).}, that is,
\begin{equation}
    \big| \ell(\hat \by_1,\by) - \ell(\hat \by_2,\by)\big|
    \leq
    M \big\| \hat\by_1 - \hat \by_2 \big\|.
\end{equation}
\end{assumption}

Let \((\tilde{\textup{P}}_{M\nu})\) denote the problem obtained by tightening the constraints in~\eqref{P:csl_variational} by \(M\nu\), and let \((\tilde{\textup{D}}_{M\nu})\) be its dual problem. We assume that the optimal value of \((\tilde{\textup{D}}_{M\nu})\) is attained by a multiplier \(\tilde \lambda^*_{M\nu}\).

\begin{figure*}[t]
    \centering
    \includegraphics[width=0.9\linewidth]{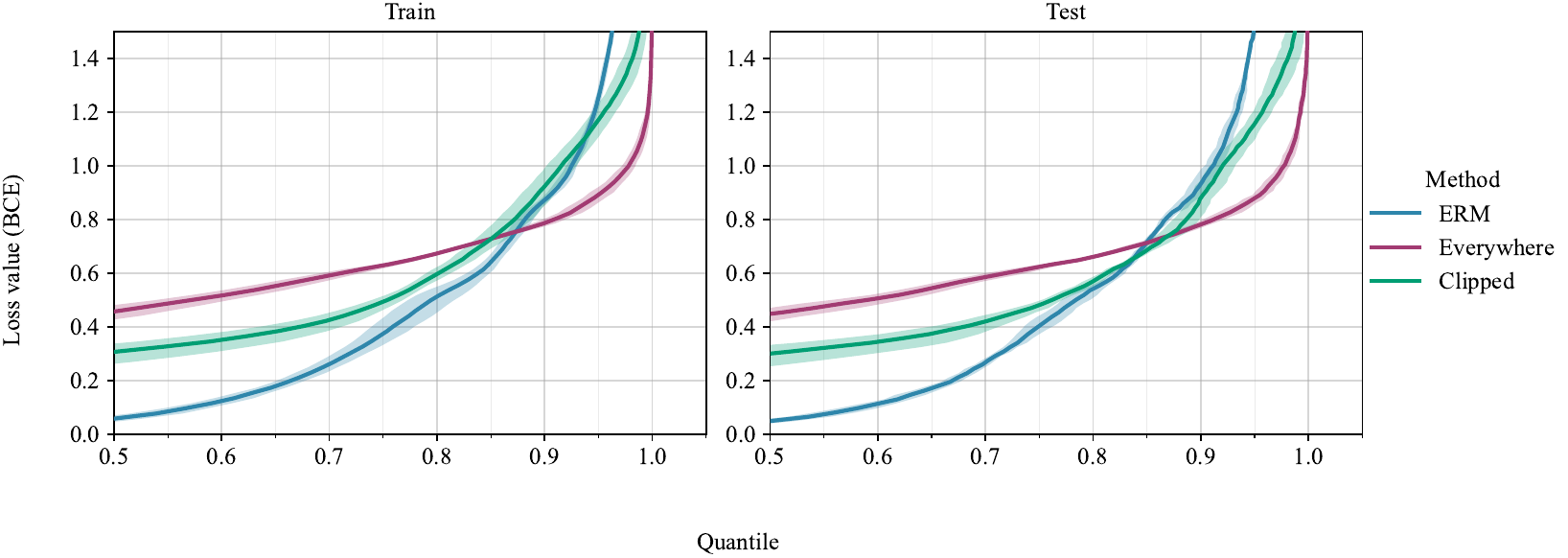}
    \caption{Loss quantile functions for \(\ell_{\mathrm{BCE}}(f_\theta(G,T),y)\) on the training and test splits of the Coding-AF task in FLORA-Bench. \textit{ERM} denotes minimization of the average loss, \textit{Everywhere} denotes optimization of the empirical dual problem, and \textit{Clipped} denotes the clipped empirical dual with $\gamma=1$. Solid curves and shaded regions represent the mean and standard deviation, respectively, across 10 random seeds. \emph{Everywhere Learning substantially reduces the upper tail of the loss distribution, preventing individual samples from incurring excessively large losses.}}

\label{fig:ecdf}
\end{figure*}
\begin{assumption}[Dual attainability]\label{ass:constraint-qual-tightened}
The dual solution $\tilde \lambda^*_{M\nu}$ of problem \((\tilde{\textup{D}}_{M\nu})\) exists.
\end{assumption}

Assumption~\ref{ass:near-univ} quantifies how well predictors in \(\mathcal H\) approximate elements of the convexified class \(\hat{\mathcal H}\), both in expectation and pointwise. Classical universal approximation results ensure that this type of condition holds for sufficiently wide neural networks; see, for example,~\cite[Theorem~2.1]{hornik1989multilayer}. Assumption~\ref{ass:lip_out} is satisfied by many standard losses, including the squared loss and negative log-likelihood. Assumption~\ref{ass:constraint-qual-tightened} strengthens Assumption~\ref{ass:dual_exis}, since it imposes attainability for the tightened problem. As with Assumption~\ref{ass:dual_exis}, this condition can be replaced by a strict feasibility assumption when working with dual variables in distribution space.

The next result shows that the approximation quality of \(\mathcal H\) controls the duality gap of~\eqref{P:primal_statistical}.


%
\begin{lemma}[Duality gap of the parametric problem]
\label{lemma:dual-gap}
Under Assumptions~\ref{ass:strong-convx-loss}--\ref{ass:constraint-qual-tightened}, it holds that
\begin{equation}
    |P^\star - D^\star|
    \;\leq\;
    M\nu \E_{\mathcal{D}}\big[\lambda^\star_{M\nu}(\bz)\big]+M\nu_0 := \Delta_{\mathrm{gap}}.
\end{equation}
\end{lemma}
\begin{proof}
    See Appendix~\ref{apx:duality_gap_proof} in the supplementary material. 
\end{proof}
Lemma~\ref{lemma:dual-gap} shows that the duality gap is controlled by how well elements of \(\hat{\mathcal H}\) can be approximated by predictors in \(\mathcal H\). Thus, richer parametrizations yield smaller duality gaps. The bound is scaled by the Lipschitz constant \(M\), which measures how perturbations in the predictor affect the loss values. In addition, the pointwise approximation error \(M\nu\) is weighted by the total mass of the optimal multiplier for the tightened problem, reflecting its sensitivity to perturbations of the constraints.

When the duality gap is bounded, the PACC-learnability guarantee of Theorem~\ref{thm:pacc} extends to the nonconvex parametric setting, up to the additional approximation term \(\Delta_{\mathrm{gap}}\) in the optimality bound.

\begin{theorem}[Near-PACC learnability of \ref{P:primal_statistical}]
\label{thm:near-pacc}
Let Assumptions~\ref{ass:unif_conv}--\ref{ass:bounded-inf-norm} and Assumptions~\ref{ass:strong-convx-loss}--\ref{ass:constraint-qual-tightened} hold, and let \(\bar f_\theta\) be the estimator returned by~\eqref{P:primal_empirical}. Then, it holds with probability at least \(1-5\delta\) that
\begin{align}\label{eq:near_pacc_obj_thm}
\left|
\mathbb{E}_{\bz \sim \mathcal D_0}\big[\ell_0(\bar f_{\theta},\bz)\big] - P^\star
\right|
&\le
  \Delta_{\mathrm{gap}} + 3\zeta_0 + \gamma\,\zeta, \\
    \Pr_{\bz\sim\mathcal D}\big[\ell(\bar f_{\theta},\bz) \le 0\big] &\geq 1-\zeta_I.
\end{align}
Therefore,~\eqref{P:primal_empirical} is a near-PACC learner for~\eqref{P:primal_statistical}.
\end{theorem}
\begin{proof}
    See Appendix~\ref{apx:nearPACC} in the supplementary material. 
\end{proof}

The optimality bound in Theorem~\ref{thm:near-pacc} consists of two terms. A statistical term inherited from dual learnability, which vanishes with the number of samples, and the approximation term \(\Delta_{\mathrm{gap}}\), which is independent of the sample size. Therefore, the theorem does not establish exact PACC learnability in general, but rather near-PACC learnability. Nevertheless, as discussed above, Lemma~\ref{lemma:dual-gap} shows that this residual error is controlled by the richness of the parametrization.

In the non convex case, establishing learnability guarantees for \eqref{P:dual_empirical} is not as direct. In particular, it is not true anymore that the estimator returned by \eqref{P:dual_empirical} coincides with the estimator returned by \eqref{P:primal_empirical}. We provide a partial result that shows that \eqref{P:dual_empirical} yields a primal minimizer that is PACC feasible, and yields a Lagrangian value which is near-PACC optimal.

\begin{proposition}\label{prop:dual_near_pacc}
Let Assumptions~\ref{ass:unif_conv}--\ref{ass:bounded-inf-norm} and Assumptions~\ref{ass:strong-convx-loss}--\ref{ass:constraint-qual-tightened} hold. Then there exists a primal minimizer associated with the empirical dual solution~\eqref{P:dual_empirical},
\begin{equation}\label{eq:estimator_pacc}
    \bar f_\theta
    \in
    \arg \min_{f_\theta \in \mathcal{H}}
    \hat L(f_\theta, \hat \lambda^*),
\end{equation}
such that, with probability at least \(1-5\delta\),
\begin{align}\label{eq:near_pacc_obj_thm}
\left|
\mathbb{E}_{\bz \sim \mathcal D_0}
\big[
\ell_0(\bar f_{\theta},\bz)
\big]
-
P^\star
\right|
&\le
\max \{ \Delta_{\mathrm{gap}}, \Delta_{\mathrm{slack}}\}
+
3\zeta_0
+
\gamma\,\zeta,
\\
\Pr_{\bz\sim\mathcal D}\big[\ell(\bar f_{\theta},\bz) \le 0\big]
&\geq
1-\zeta_I,
\end{align}
where
\begin{equation}
    \Delta_{\mathrm{slack}}
    =
    -\frac{1}{N}
    \sum_{n=1}^N
    \hat \lambda^*_n
    \ell(\bar f_\theta, \bz_n).
\end{equation}
\end{proposition}
\begin{proof}
    See Appendix~\ref{apx:nearPACC-dual} in the supplementary material. 
\end{proof}

Proposition~\ref{prop:dual_near_pacc} provides a near-optimality guarantee for the empirical dual learning rule~\eqref{P:dual_empirical}. We emphasize, however, that the guarantee ensures that at least one Lagrangian minimizer satisfies the near-PACC bounds, but not necessarily every minimizer. This is a standard primal-recovery issue in nonconvex duality. Moreover, the constant term in the optimality bound is now the maximum between the duality gap and the empirical complementary-slackness error of the selected minimizer. The duality-gap term is controlled by Lemma~\ref{lemma:dual-gap}, as discussed above. The slackness term is not bounded by our analysis, but it is directly computable after solving the empirical dual problem and therefore provides a practical diagnostic for the quality of primal recovery.

We remark that the same strategy which bound the duality gap using \eqref{P:csl_variational} could be applied to bound the slackness term and provide guarantees over all Lagrangian minimizers. Such results have been established for expected-value constraints in~\cite{elenter2024nearoptimalsolutionsconstrainedlearning}. Extending this analysis to the present pointwise-constrained setting requires a separate treatment and is left for future work. Nevertheless, in the next section, we demonstrate empirically that dual methods remain effective in practice, even for a nonconvex task.


\section{Numerical Experiments}
\label{sec:num_exps}

\begin{figure*}[t]
    \centering
    \includegraphics[width=0.9\linewidth]{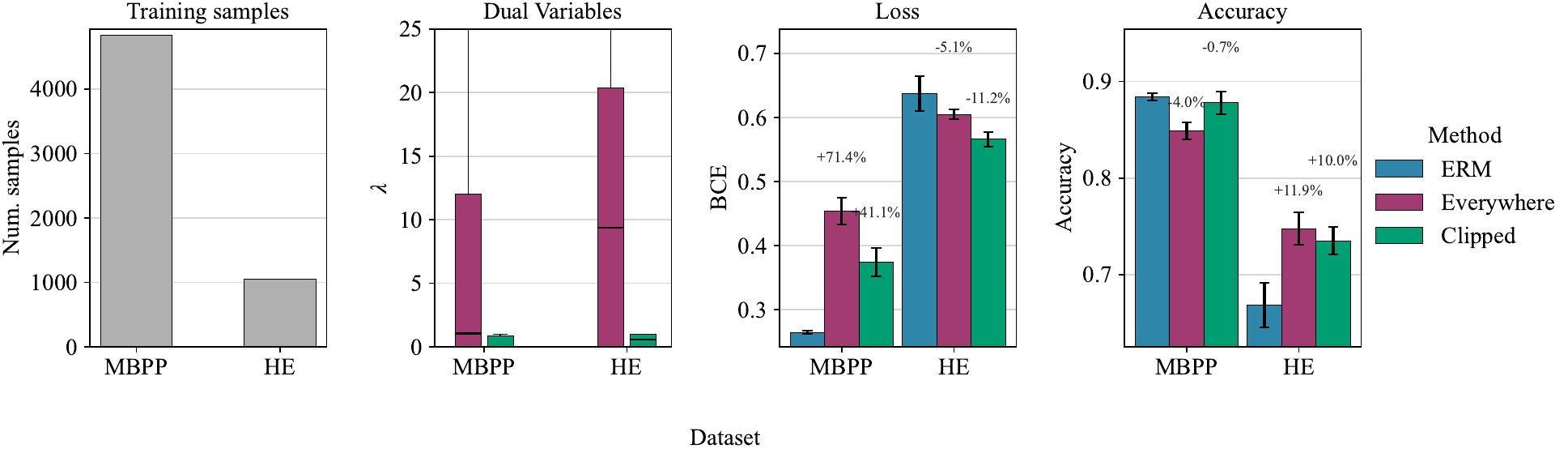}
    \caption{Number of training samples, dual-variable values at the end of training, test loss, and test accuracy for the Coding-AF task in FLORA-Bench~\cite{flora-bench}, reported separately for the MBPP and HumanEval (HE) data sources. For the \textit{Everywhere} and \textit{Clipped} methods, we also report the relative changes in loss and accuracy with respect to ERM. \emph{Everywhere Learning reduces performance disparities across data sources, substantially improving accuracy on the underrepresented HumanEval samples while largely preserving performance on MBPP.}}
\label{fig:barplots}
\end{figure*}





We instantiate Everywhere Learning for robust agentic-workflow classification using the FLORA-Bench framework~\cite{flora-bench}.

\noindent
\textbf{Dataset.}
Each input consists of a workflow--task pair $(G,T)$, where $T$ is a coding-problem prompt and $G$ is an agentic workflow. An agentic workflow is a structured multi-agent system in which specialized agents collaborate through task dependencies to solve a problem. We represent each workflow as a directed graph whose nodes correspond to agents, characterized by their prompts, and whose edges encode information flow or task dependencies. The label $y\in\{0,1\}$ indicates whether the workflow successfully solves the task. Coding tasks are drawn from the HumanEval~\cite{humaneval} and MBPP~\cite{mbpp} benchmarks, while workflows are generated using AFLOW~\cite{zhang2024aflow}.

\noindent
\textbf{Architecture.}
Our predictor follows the hybrid graph neural network and sentence transformer architecture introduced in FLORA-Bench~\cite{flora-bench}.

\noindent
\textbf{Formulation.}
We seek uniformly robust classification by requiring the binary cross-entropy loss to remain below a prescribed tolerance for almost every sample, while minimizing its expected value:
\begin{equation}
\begin{aligned}
\label{eq:flora}
    \min_{\theta}\quad &
    \mathbb{E}_{(G,T,y)\sim\mathcal{D}}
    \left[
        \ell_{\mathrm{BCE}}\bigl(f_\theta(G,T),y\bigr)
    \right]
    \\
    \text{s.t.}\quad &
    \ell_{\mathrm{BCE}}\bigl(f_\theta(G,T),y\bigr)
    \leq c,
    \qquad \mathcal{D}\text{-a.e.}
\end{aligned}
\end{equation}
where
\begin{equation}
    \ell_{\mathrm{BCE}}(\hat{y},y)
    =
    -y\log(\hat{y})-(1-y)\log(1-\hat{y}).
\end{equation}
\textit{Everywhere} denotes the solution of problem~\eqref{eq:flora}; \textit{ERM} minimizes only its objective, without enforcing the pointwise constraints; and \textit{Clipped} solves~\eqref{eq:flora} while imposing an upper bound on the dual variables.

\noindent

\medskip

\textbf{(C1) Dual variables are larger for underrepresented samples.}
Figure~\ref{fig:barplots} reveals a substantial imbalance in the training set, with HumanEval containing considerably fewer samples than MBPP. Consistent with our theoretical analysis, the \textit{everywhere} method assigns larger dual variables to the underrepresented HumanEval samples. Because the pointwise constraints impose the same performance requirement on every sample, the optimal sensitivity distribution \(P_{\lambda}^{*}\) must allocate sufficient mass to both datasets. However, since HumanEval samples have lower probability under the empirical data distribution, they require larger dual variables according to \eqref{eq:dist-lambda-optimum}. This empirical behavior supports the theoretical prediction developed in Subsection~\ref{sec:ass-disc} and illustrated in Figure~\ref{fig:distr-learning}, where dual variables compensate for low data probability by assigning greater importance to underrepresented samples.

\medskip

\textbf{(C2) Everywhere Learning prevents large loss violations.}
Figure~\ref{fig:ecdf} presents the empirical loss quantile functions of the unconstrained and constrained methods on the test set. At high quantiles, the curves corresponding to the constrained methods lie consistently below that of ERM, demonstrating a substantial reduction in the upper tail of the loss distribution. Thus, pointwise constraints prevent individual samples from incurring disproportionately large losses. This effect is particularly pronounced for the \textit{everywhere} method, for which nearly all test samples attain a loss below \(1.5\). Conversely, ERM achieves lower losses at small quantiles, where samples are already correctly classified and further loss reductions provide limited practical benefit. Rather than allocating additional model capacity to these already well-performing samples, Everywhere Learning redirects it toward the hardest samples, reducing large violations and producing a more uniform loss distribution.

\medskip

\textbf{(C3) Everywhere Learning improves performance uniformity.}
Figure~\ref{fig:barplots} shows that the imbalance between MBPP and HumanEval translates into disparate test performance under ERM as samples from the underrepresented HumanEval dataset exhibit both higher losses and lower accuracy. By reducing the upper tail of the loss distribution, Everywhere Learning substantially narrows this performance gap. In particular, for the clipped method it reduces the loss on HumanEval samples by \(11.2\%\) relative to the unconstrained model, resulting in a \(10.0\%\) improvement in accuracy. This gain is accompanied by a \(41.1\%\) increase in loss on MBPP, but only a \(0.7\%\) decrease in its accuracy. This asymmetry confirms that further reducing the loss of already correctly classified samples provides little practical benefit, whereas reallocating model capacity toward high-loss and misclassified samples can produce substantial accuracy gains. 

\section{Conclusions}

We introduced \emph{everywhere learning}, a constrained learning framework in which requirements are enforced almost everywhere over the data distribution. This formulation separates objectives from requirements and prevents violations in some regions of the input space from being compensated by improvements elsewhere. Since the data distribution is unknown and only accessible through samples, we studied when these statistical constrained problems can be approximated from finite data. Our results show that uniform convergence of the objective and constraint losses is not sufficient to guarantee learnability. The key quantity is instead the dual sensitivity of the problem. We formalized this through dual PAC learnability and showed that boundedness of the optimal multiplier is sufficient for learning the dual problem. This condition limits the extent to which regions that are unlikely under the data distribution can  impact on the optimum. When the condition does not hold, we showed that it can be enforced by restricting the dual domain, which is equivalent to an \(L^1\)-penalized sparse relaxation of the pointwise constraints and can be implemented empirically by clipping the dual variables. We then related dual learnability back to primal learnability. In convex settings, strong duality transfers the dual guarantees directly to PACC learnability. For nonconvex hypothesis classes, we obtained near-PACC guarantees controlled by the approximation quality of the parametrized class and by the primal recovery error of the dual rule. The experiments support these conclusions, showing that pointwise constraints and sparse dual regularization can improve generalization and recover high-quality solutions even in nonconvex settings.

\bibliographystyle{IEEEtran}
\bibliography{refs}

\appendices

\section{Empirical problems are a relaxation}
\label{apx:emp_is_relax}
Given \(N\) samples \(S=\{(\bx_n,\by_n)\}_{n=1}^{N}\) drawn from \(\mathcal{D}\), the almost-everywhere nature of the constraint in ~\eqref{P:primal_statistical} ensures that any feasible point must satisfies the constraint for all samples \((\bx_n,\by_n)\in S\). Therefore, every feasible point of~\eqref{P:primal_statistical} is feasible for~\eqref{P:primal_empirical}. We refer to this property as the empirical problem being a \textit{relaxation} of the statistical problem. As a consequence, one can obtain an upper bound on the statistical optimum using only Assumption~\ref{ass:unif_conv}.
\begin{lemma}[Supra-optimality of \ref{P:primal_empirical}]\label{thm:empirical_is_relaxation}
Let Assumption~\ref{ass:unif_conv} hold. Then, for every \(\delta\in(0,1)\) and every distribution \(\mathcal{D}\), the rule \eqref{P:primal_empirical} returns an estimator \(f_{\theta}\in\mathcal{H}\) such that, with probability at least \(1-2\delta\),
\begin{enumerate}
\item \emph{Probably approximately supra-optimal:}
\begin{equation}
\mathbb{E}_{(x,y)\sim \mathcal{D}_0}\big[\ell_0(f_{\theta}(x),y)\big]
\;\le\; P^\star + 2\zeta_0(N,\delta),
\end{equation}
\item \emph{Probably approximately feasible:}
\begin{equation}
Pr\big[\ell(f_{\theta}(\bx),\by) \le c_i\big] \geq 1-\zeta_{I}(N,\delta).
\end{equation}
\end{enumerate}
\end{lemma}

\begin{proof}

\noindent
\textbf{Feasibility}. 
We first prove that $\hat f_\theta$ is feasible for \ref{P:primal_empirical}. Suppose it is not; then $\hat P = \infty$. But $P \leq \infty$ imply that $\exists f_\theta$ feasible for \ref{P:primal_statistical}. As the constraint holds $D-$a.e., $f_\theta$ must be feasible for \ref{P:primal_empirical}. Combined with the boundedness of $\ell_0$ it implies that $f_\theta$ is a feasible point with bounded objective value, violating the assumption that $\hat f_\theta$ is a solution of \ref{P:primal_empirical}. Therefore $\hat f_\theta$ is feasible for \ref{P:primal_empirical} and thus

\begin{equation}
\frac{1}{N}\sum_{n=1}^{N}\mathbb{I}\!\left\{\ell\!\big(f_\theta(x_{n}),y_{n}\big)\le c\right\} = 1.
\end{equation}

The uniform convergence of $\mathcal{H}$ with respect to $\mathbb{I}\!\left\{\ell\!\big(f_\theta(x),y\big)\le c\right\}$ ensures that with probability $1-\delta$,

\begin{align}
 Pr\big[\ell&(f_{\theta}(\bx),\by) \le c\big] =\mathbb{E}_{(x,y)\sim \mathcal{D}}\!\left[\mathbb{I}\!\left\{\ell\!\big(f_\theta(x),y\big)\le c\right\}\right] \\
 &\geq  \frac{1}{N}\sum_{n=1}^{N}\mathbb{I}\!\left\{\ell\!\big(f_\theta(x_{n}),y_{n}\big)\le c\right\} - \zeta_{I}(N,\delta) \\ 
 & = 1 - \zeta_{I}(N,\delta).
\end{align}

\noindent
\textbf{Supra-optimality}. Let $f_\theta^*$ be the solution of $\ref{P:primal_statistical}$. Then, with probability $1-2\delta$,

\begin{align}
    \E_{\mathcal D_0} \!\Big[ \ell_0\big(\hat f_{\theta}(\bx),y \big) \Big] &\leq  \frac{1}{N_0} \sum_{n = 1}^{N_0} \ell_0\big( \hat f_{\theta}(\bx_{n}), y_{n} \big) + \zeta_0(N,\delta) \label{eq:appxA-unif-conv-1}\\
    &\leq \frac{1}{N_0} \sum_{n = 1}^{N_0} \ell_0\big(f_{\theta}(\bx_{n}), y_{n} \big) + \zeta_0(N,\delta) \label{eq:appxA-suboptimal}\\
    &\leq \E_{\mathcal D_0} \!\Big[ \ell_0\big(f_{\theta}^*(\bx),y \big) \Big] + 2\zeta_0(N,\delta) \label{eq:appxA-unif-conv-2}\\
    &= P^* + 2\zeta_0(N,\delta)
\end{align}
where equations \eqref{eq:appxA-unif-conv-1} and \eqref{eq:appxA-unif-conv-2} result from uniform convergance of $\ell_0$ whereas \eqref{eq:appxA-suboptimal} result from the suboptimallity of $f_\theta^*$ for problem \ref{P:primal_empirical}.
\end{proof}

This result may appear satisfactory because it guarantees probable approximate feasibility while controlling how much worse the objective can be relative to the optimal value. However, it provides an incomplete characterization of the estimator as a solution to the statistical problem, since it does not bound the improvement in objective achieved solely due to the relaxation of the constraints. In particular, if no other assumption than uniform convergence holds, the solution may hold arbitrarily slow rates of convergence to the statistical optimum, as shown in Example~\ref{ex:counter-uc-lambda}, which prevent learnability of the problem.

\section{Radamacher Complexity and UC rate derivation} \label{apx:radam}
In Assumption~\ref{ass:unif_conv} we assume the existence of an uniform convergence rate for the constraint and objective losses. Rademacher complexity provides one standard way of verifying such an assumption by quantifying how rich the loss class induced by the hypothesis class is.

\begin{definition}[Rademacher complexity]
Let \(S=(z_1,\dots,z_N)\sim \mathcal{D}^N\) be \(N\) i.i.d. samples. Let \(\mathcal{L}\) be a class of real-valued functions on \(\mathcal{Z}\). The empirical Rademacher complexity of \(\mathcal{L}\) on \(S\) is
\begin{equation}
\hat{\mathcal{R}}(\mathcal{L}\circ S)
\;:=\;
\frac{1}{N}\,
\mathbb{E}_{\sigma}\left[
\sup_{h \in \mathcal{L}}
\sum_{i=1}^N \sigma_i h(z_i)
\right],
\end{equation}
where \(\sigma_1,\dots,\sigma_N\) are independent Rademacher random variables satisfying
\[
\mathbb{P}[\sigma_i=1]
=
\mathbb{P}[\sigma_i=-1]
=
\tfrac12 .
\]
The expected Rademacher complexity of \(\mathcal{L}\) is
\begin{equation}
\mathcal{R}_N(\mathcal{L})
\;:=\;
\E_S \left[
\hat{\mathcal{R}}(\mathcal{L}\circ S)
\right].
\end{equation}
\end{definition}

In our setting, the relevant function class is the class of losses generated by hypotheses in \(\mathcal{H}\). For a loss function \(\ell\), define
\begin{equation}
\mathcal{L}_{\ell}(\mathcal{H})
:=
\left\{
z=(\bx,y)\mapsto \ell\big(f_\theta(\bx),y\big)
\;:\;
f_\theta\in\mathcal{H}
\right\}.
\end{equation}
The Rademacher complexity of \(\mathcal{L}_{\ell}(\mathcal{H})\) measures how well the losses induced by hypotheses in \(\mathcal{H}\) can correlate with independent random signs on a sample. If this correlation is small, then the class cannot fit random noise well, which limits the gap between empirical and statistical averages uniformly over \(f_\theta\in\mathcal{H}\).

The following standard result makes this connection precise.

\begin{proposition}[Vanishing Rademacher complexity implies uniform convergence, Theorem 26.5~\cite{shalev2014understanding}]
\label{prop:rad-imply-uc}
Assume that every function in \(\mathcal{L}_{\ell}(\mathcal{H})\) is bounded in \([0,B]\). If
\[
\mathcal{R}_N\big(\mathcal{L}_{\ell}(\mathcal{H})\big)
\to 0
\]
for every distribution \(\mathcal{D}\), then \(\mathcal{H}\) satisfies the uniform convergence property in Definition~\ref{def:uniform-conv} with rate
\begin{equation}\label{apx:rate-rada}
\zeta_R(N,\delta)
:=
2\mathcal{R}_N\big(\mathcal{L}_{\ell}(\mathcal{H})\big)
+
B\sqrt{\frac{\log(2/\delta)}{2N}} .
\end{equation}
\end{proposition}

Thus, whenever the Rademacher complexity of the induced loss class vanishes with \(N\), the empirical risk converges uniformly to the statistical risk over all \(f_\theta\in\mathcal{H}\) with the rate given in \eqref{apx:rate-rada}. This is precisely the property required in Definition~\ref{def:uniform-conv}.


\section{Proof of Lemma \ref{prop:dual_equiv}} \label{apx:dual_equiv}
Fix \(\epsilon,\delta\in(0,1)\). Let \(\hat\lambda^\star\in\mathbb{R}_+^N\) be an optimal solution of the empirical dual problem~\eqref{P:dual_empirical}, and let  $\hat f_\theta^\star$ be the corresponding minimizer.
Then, 
\begin{equation} 
\hat D^\star = \hat L(\hat f_\theta^\star,\hat\lambda^\star). 
\end{equation}

For each sample \(\bz_n\), let \(B_\rho(\bz_n)\) denote a ball of radius \(\rho\) centered at \(\bz_n\). Define 

\begin{equation} 
\kappa_{n,\rho}(\bz) := \frac{\mathbb{I}\{\bz\in B_\rho(\bz_n)\}} {\mathcal D(B_\rho(\bz_n))}. 
\end{equation} 

This function is nonnegative and normalized with respect to \(\mathcal D\), since 
\begin{equation} 
\mathbb{E}_{\bz\sim\mathcal D} \big[ \kappa_{n,\rho}(\bz) \big] = 1. 
\end{equation} 

It is therefore a \(\mathcal D\)-normalized bump around \(\bz_n\), whose mass concentrates at the sample point as \(\rho\to0\). 

Define the functional multiplier \begin{equation} \hat\lambda_\rho(\bz) := \frac{1}{N} \sum_{n=1}^N \hat\lambda_n^\star \kappa_{n,\rho}(\bz). \end{equation}

As \(\rho\to0\), the bump \(\kappa_{n,\rho}\) concentrates around \(\bz_n\). Hence,
\begin{equation}
\lim_{\rho\to0}
\mathbb{E}_{\bz\sim\mathcal D}
\big[
\kappa_{n,\rho}(\bz)\ell(\hat f_\theta^\star,\bz)
\big]
=
\ell(\hat f_\theta^\star,\bz_n).
\end{equation}
Therefore,
\begin{equation}
\lim_{\rho\to0}
L(\hat f_\theta^\star,\hat\lambda_\rho)
=
\mathbb{E}_{\bz\sim\mathcal D}
\big[
\ell_0(\hat f_\theta^\star,\bz)
\big]
+
\frac{1}{N}
\sum_{n=1}^N
\hat\lambda_n^\star
\ell(\hat f_\theta^\star,\bz_n).
\end{equation}
Comparing this expression with the empirical Lagrangian gives

\begin{multline}
\lim_{\rho\to0}
\left|
L(\hat f_\theta^\star,\hat\lambda_\rho)
-
\hat L(\hat f_\theta^\star,\hat\lambda^\star)
\right|
=
|
\mathbb{E}_{\bz\sim\mathcal D_0}
\big[
\ell_0(\hat f_\theta^\star,\bz)
\big]
- \\
\frac{1}{N_0}
\sum_{n=1}^{N_0}
\ell_0(\hat f_\theta^\star,\bz_n) |.
\end{multline}

By Assumption~\ref{ass:unif_conv}, with probability at least \(1-\delta/2\),
\begin{equation}
\left|
\mathbb{E}_{\bz\sim\mathcal D_0}
\big[
\ell_0(\hat f_\theta^\star,\bz)
\big]
-
\frac{1}{N}
\sum_{n=1}^N
\ell_0(\hat f_\theta^\star,\bz_n)
\right|
\leq
\zeta_0(N,\delta/2).
\end{equation}
Taking \(N\) large enough so that
\begin{equation}
\zeta_0(N,\delta/2)
\leq
\epsilon/3,
\end{equation}
and then taking \(\rho\) small enough, we obtain
\begin{equation}
\left|
L(\hat f_\theta^\star,\hat\lambda_\rho)
-
\hat L(\hat f_\theta^\star,\hat\lambda^\star)
\right|
\leq
2\epsilon/3.
\end{equation}

From the premise it holds for \(\epsilon/3\) and \(\delta/2\), that for \(N\) sufficiently large, with probability at least \(1-\delta/2\), 
\begin{equation} 
\left|D^\star-\hat D^\star\right| \leq \epsilon/3 . 
\end{equation}

Combining the last two equations yields, with probability at least $1-\delta$,
\begin{align}
\left|
D^\star
-
L(\hat f_\theta^\star,\hat\lambda_\rho)
\right|
&\leq
\left|
D^\star-\hat D^\star
\right|
+
\left|
\hat L(\hat f_\theta^\star,\hat\lambda^\star)
-
L(\hat f_\theta^\star,\hat\lambda_\rho)
\right|
\\ &\leq
\epsilon .
\end{align}

As this holds for any $\epsilon$ and  $\delta$, dual PAC learnability holds for  the pair \((\hat f_\theta^\star,\hat\lambda_\rho)\) returned by \eqref{P:dual_empirical}.

\section{Proof of Lemma \ref{lemma:suff-cond-lags}} \label{apx:suff-cond}
Define the set of primal minimizers for any arbitrary $\lambda$ and $\hat \lambda$ as
\begin{equation}
	F_\theta^*(\lambda) = \arg \min_{f_\theta \in \mathcal{H}} L(f_\theta,\lambda)
	\;;\;
	\hat F_\theta (\hat{\lambda}) = \arg \min_{f_\theta \in \mathcal{H}} \hat{L}(f_\theta,\hat{\lambda}).
\end{equation}
for the Lagrangians defined in~\eqref{E:lagrangian_csl_param} and~\eqref{E:empirical_lagrangian} respectively.

\noindent
\textbf{Lower bound.}
Let \(\hat\lambda^\star\in\mathbb{R}_+^N\) be an optimal solution of the empirical dual problem, and let $\hat f_\theta^\star \in \hat F_\theta (\hat{\lambda}^*)$.

For each sample \(\bz_n\), let \(B_\rho(\bz_n)\) be a ball of radius \(\rho\) centered at \(\bz_n\), and define
\begin{equation}
\kappa_{n,\rho}(\bz)
:=
\frac{\mathbb{I}\{\bz\in B_\rho(\bz_n)\}}
{\mathcal D(B_\rho(\bz_n))}.
\end{equation}
Thus, \(\kappa_{n,\rho}\) is nonnegative, supported on \(B_\rho(\bz_n)\), and satisfies
\begin{equation}
\mathbb{E}_{\bz\sim\mathcal D}
\big[
\kappa_{n,\rho}(\bz)
\big]
=
1.
\end{equation}
Define
\begin{equation}
\hat\lambda_\rho(\bz)
:=
\frac{1}{N}
\sum_{n=1}^N
\hat\lambda_n^\star
\kappa_{n,\rho}(\bz).
\end{equation}
Let $\hat f_{\theta,\rho}
\in  F_\theta (\hat\lambda_\rho)$. Suboptimality of  \(\hat\lambda_\rho\) yields,
\begin{equation}
D^\star-\hat D^\star
\geq
L(\hat f_{\theta,\rho},\hat\lambda_\rho)
-
\hat L(\hat f_\theta^\star,\hat\lambda^\star).
\end{equation}
Since \(\hat f_\theta^\star\) minimizes \(\hat L(\cdot,\hat\lambda^\star)\), we have
\begin{equation}
D^\star-\hat D^\star
\geq
L(\hat f_{\theta,\rho},\hat\lambda_\rho)
-
\hat L(\hat f_{\theta,\rho},\hat\lambda^\star).
\end{equation}
By definition of the Lagrangians we get
\begin{multline}
L(\hat f_{\theta,\rho},\hat\lambda_\rho)
-
\hat L(\hat f_{\theta,\rho},\hat\lambda^\star) \\
= 
\mathbb{E}_{\bz\sim\mathcal D_0}
\big[
\ell_0(\hat f_{\theta,\rho},\bz)
\big]
-
\frac{1}{N_0}
\sum_{n=1}^{N_0}
\ell_0(\hat f_{\theta,\rho},\bz_n)
\\
\quad
+
\frac{1}{N}
\sum_{n=1}^N
\hat\lambda_n^\star
\left(
\mathbb{E}_{\bz\sim\mathcal D}
\big[
\kappa_{n,\rho}(\bz)
\ell(\hat f_{\theta,\rho},\bz)
\big]
-
\ell(\hat f_{\theta,\rho},\bz_n)
\right).
\end{multline}

By uniform convergence of the objective loss, with probability at least \(1-\delta\),
\begin{equation}
\mathbb{E}_{\bz\sim\mathcal D_0}
\big[
\ell_0(\hat f_{\theta,\rho},\bz)
\big]
-
\frac{1}{N_0}
\sum_{n=1}^{N_0}
\ell_0(\hat f_{\theta,\rho},\bz_n)
\geq
-\zeta_0(N,\delta).
\end{equation}

By construction, as \(\rho\to0\) 
\begin{equation}
    \mathbb{E}_{\bz\sim\mathcal D}
\big[
\kappa_{n,\rho}(\bz)
\ell(\hat f_{\theta,\rho},\bz)
\big]
\to
\ell(\hat f_{\theta,\rho},\bz_n)
\end{equation}

and thus the term 
\begin{equation}
\eta_\rho = \frac{1}{N}
\sum_{n=1}^N
\hat\lambda_n^\star
\left(
\mathbb{E}_{\bz\sim\mathcal D}
\big[
\kappa_{n,\rho}(\bz)
\ell(\hat f_{\theta,\rho},\bz)
\big]
-
\ell(\hat f_{\theta,\rho},\bz_n)
\right)
\end{equation}
can be made arbitrarily small.

Hence,
\begin{equation}
D^\star-\hat D^\star
\geq
-\zeta_0(N,\delta)-\eta_\rho.
\end{equation}
Taking the limit \(\rho\to0\) yields
\begin{equation}\label{eq:apx-low-bound-final}
D^\star-\hat D^\star
\geq
-\zeta_0(N,\delta).
\end{equation}

\noindent
\textbf{Upper Bound.} Let $\tilde f_{\theta} \in  \hat F^*_\theta(\tilde \lambda)$, with $\tilde \lambda$ defined in the hypothesis of the theorem.
Sub-optimality of $\tilde \lambda$ yields
\begin{align}\label{eq:apx-low-bound-t2-0}
	L(f_\theta^*,\lambda^*) - \hat L(\hat f_\theta^*,\hat \lambda) &\leq  L(f_\theta^*,\lambda^*) - \hat L(\tilde f_\theta,\tilde \lambda),
\end{align}
whereas sub-optimality of $\tilde f_\theta$ yields
\begin{equation}
     L(f_\theta^*,\lambda^*) - \hat L(\tilde f_\theta,\tilde \lambda) \leq L(\tilde f_\theta,\lambda^*) - \hat L(\tilde f_\theta,\tilde \lambda).
\end{equation}
Under the hypothesis in \eqref{eq:unif-conv-lagrangians} it holds, with probability $1-\delta$,
\begin{equation}\label{eq:apx-low-bound-t2}
    L(\tilde f_\theta,\lambda^*) - \hat L(\tilde f_\theta,\tilde \lambda) \leq \zeta_L(N,\delta).
\end{equation}
Combining the inequalities in \eqref{eq:apx-low-bound-t2-0}-\eqref{eq:apx-low-bound-t2} yields, with probability at least $1-\delta$,
\begin{equation}\label{eq:apx-upper-bound-final}
D^* - \hat D^* = L(f_\theta^*,\lambda^*) - \hat L(\hat f_\theta^*,\hat \lambda) \leq \zeta_L(N,\delta)
\end{equation}
\noindent
\textbf{Union Bound}
Finally, applying a union bound argument between \eqref{eq:apx-low-bound-final} and  \eqref{eq:apx-upper-bound-final} gives, with probability $1-2\delta$,
\begin{equation}
    |L(f_\theta^*,\lambda^*) - \hat L(\hat f_\theta^*, \hat \lambda^*)|  \leq  \zeta_0(N,\delta) + \zeta_L(N,\delta),
\end{equation}
which completes the proof.

\section{Proof of Lemma \ref{lemma:dual-learn}} \label{apx:dual_learning_2}
\noindent\textbf{Preliminary lemma}. We first state and prove an auxiliary lemma.

\begin{lemma}\label{lemma:rademacher_bound}
Let $\mathcal{H}$ be a hypothesis class, let $\ell:\mathcal{H}\times\mathcal{Z}\to\mathbb{R}$ be a loss, and let  $g:\mathcal{Z}\to\mathbb{R}_+$ be such that
$\|g\|_\infty<\infty$ almost surely. Then
\begin{equation}\label{eq:apx-lemma-rad}
\mathcal{R}_N\big(\mathcal{H}\circ g\ell\big)
\;\le\;
\|g\|_\infty\;
\mathcal{R}_N\big(\mathcal{H}\circ \ell\big).
\end{equation}

Moreover, if $\ell$ has vanishing Rademacher complexity, then so does $g\ell$, and with probability $1-\delta$
\begin{align}\label{eq:apx-lemma-rad-2}
\left|\mathbb{E}\left[g(z)\ell(f, z)\right]-\frac{1}{n} \sum_{i=1}^n g(z_i)\ell(f, z_i)\right| \leq \|g\|_\infty \zeta_R(N),
\end{align}
where $\zeta_R(N)$ is the Rademacher rate for $\ell$ as provided in Proposition \ref{prop:rad-imply-uc}.
\end{lemma}

\begin{proof}

Fix a sample $S=(\bz_1,\dots,z_N)$ and define the functions $\varphi_i: \mathbb{R} \rightarrow \mathbb{R}$ by $\varphi_i(t):= g(\bz_i) t$. Then it follows from the definition that
\begin{align}
    \hat{\mathcal{R}}_N \big( \mathcal{H} \circ g\ell \circ S \big) &= \,\mathbb{E}_{\sigma}\left[\frac{1}{N} \sup_{f_\theta \in\mathcal{H}} \sum_{i=1}^N \sigma_i g(\bz_i)\ell(f_\theta, z_i)\right] \\
    &= \,\mathbb{E}_{\sigma}\left[\frac{1}{N}\sup_{f_\theta\in\mathcal{H}} \sum_{i=1}^N \sigma_i \varphi_i(\ell(f_\theta,z_i))\right].
\end{align}
As each $\varphi_i$ is Lipchitz continuous with constants $L_i \leq  \|g\|_{\infty}$, Ledoux-talagrand contraction theorem~\cite[Theorem 4.12]{ledoux1991probability} yields
\begin{align}
\hat{\mathcal{R}}_N \big( \mathcal{H} \circ g\ell \circ S \big) &= \,\mathbb{E}_{\sigma}\left[ \frac{1}{N} \sup_{f_\theta \in\mathcal{H}} \sum_{i=1}^N \sigma_i \varphi_i(\ell(f_\theta, z_i))\right] \\
&\leq \|g\|_\infty\,\mathbb{E}_{\sigma}\left[\frac{1}{N}\sup_{f_\theta \in\mathcal{H}} \sum_{i=1}^N \sigma_i \ell(f_\theta, z_i)\right] \\
&= \|g\|_\infty \hat{\mathcal{R}}_N \big( \mathcal{H} \circ \ell \circ S \big)
\end{align}

Taking expectation over $S\sim\mathcal{D}^N$ on both sides results in \eqref{eq:apx-lemma-rad}. It is clear that a vanishing complexity for $\ell$ results in a vanishing complexity for $g\ell$. Finally, $\ell$ being bounded by $B$ implies $g\ell$ being bounded by $B \|g\|_\infty$. Then, Proposition~\ref{prop:rad-imply-uc} yields 
\begin{multline}
\left|\mathbb{E}\left[g(z)\ell(f, z)\right]-\frac{1}{n} \sum_{i=1}^n g(z_i)\ell(f, z_i)\right| \leq \\ 
2\mathcal{R}_N\big(\mathcal{H}\circ g\ell\big) + B\|g\|_\infty\sqrt{\frac{\log 2 / \delta}{2 N}} \leq \|g\|_\infty \zeta_R(N).
\end{multline}

\end{proof}

\noindent\textbf{Proof of Lemma \ref{lemma:dual-learn}}. Assumption \ref{ass:dual_exis} ensures that $\lambda^*\in L_1^+(\mathcal{D})$ and $\hat{\lambda}^*\in \reals^N$ exist. Let $\tilde \lambda \in \reals^N$ be the evaluation of $\lambda^*$ on the sample set $S$,
\begin{equation}
    \tilde \lambda_n = \lambda^*(\bz_n).
\end{equation}
By Assumption~\ref{ass:bounded-inf-norm}, $|\lambda^*|_\infty = \gamma < \infty$ and thus  $\tilde \lambda \in [0,\gamma]^N$. The Lagrangian difference is, by definition
\begin{multline}\label{apx:lag_diffrences_up_bound-2}
L(f_\theta,\lambda^*) - \hat L(f_\theta,\tilde \lambda) = \\
\Big[ \E_{\bz \sim \mathcal D_0} \!\Big[ \ell_0\big(f_\theta, \bz \big) \Big] - \frac{1}{N_0} \sum_{n = 1}^{N_0} \ell_0\big( f_\theta,\bz_{n}\big) \Big] \\
+ \Big[ \E_{\bz \sim \mathcal D } [\lambda^*(\bz) \ell\big( f_\theta, \bz \big)]  - \frac{1}{N} \sum_{n = 1}^{N} \lambda^*(\bz_n) \ell\big( f_\theta,\bz_n \big) \Big].
\end{multline}
For the first term in \eqref{apx:lag_diffrences_up_bound-2}, Assumption \ref{ass:unif_conv} ensures that uniform convergence with respect to $\ell_0$ holds, and thus with probability $1-\delta$,
\begin{equation}\label{eq:apx-obj-bound-2}
\big| \E_{\bz \sim \mathcal D_0} \!\Big[ \ell_0\big(f_\theta, \bz \big) \Big] - \frac{1}{N_0} \sum_{n = 1}^{N_0} \ell_0\big( f_\theta,\bz_{N}\big)  \big| \leq \zeta_0(N,\delta).
\end{equation}
For the second term in \eqref{apx:lag_diffrences_up_bound-2}, the bound \(||\lambda^*||_\infty \le \gamma<\infty\), together with the vanishing Rademacher complexity of \(\ell\) implied by Assumption~\ref{ass:unif_conv}, ensures that the conditions of Lemma~\ref{lemma:rademacher_bound} are satisfied.
 Thus, with probability $1-\delta$,
\begin{equation}\label{eq:apx-const-bound-2}
\Big| \E_{\bz \sim \mathcal D } [\lambda^*(\bz) \ell\big( f_\theta, \bz \big)]  - \frac{1}{N}\sum_{n = 1}^{N} \lambda^*(\bz_n) \ell\big( f_\theta,\bz_n \big)\Big| \leq \gamma \zeta(N,\delta).
\end{equation}
Triangle inequality and a union bound argument for \eqref{eq:apx-obj-bound-2}-\eqref{eq:apx-const-bound-2} yield, with probability $1-2\delta$,
\begin{equation}
    |L(f_\theta,\lambda^*) - \hat L(\hat f_\theta,\tilde \lambda)|  \leq  \zeta_0(N,\delta) + \gamma \zeta(N,\delta),
\end{equation}
which completes the proof.

\section{Proof of Proposition \ref{prop:duality_res}} \label{apx:dual_res}
\begin{proof}
The Lagrangian \(L(f_\theta,u,\lambda): \mathcal{H}\times L_1(\mathcal{\mathcal{D}})\times L_1^+(\mathcal{\mathcal{D}}) \to \reals\) of problem \eqref{P:primal_res} is,
\begin{multline}\label{eq:lagrangian_res}
    L(f_\theta,u,\lambda) = 
    \E_{\bz\sim\mathcal D_0}\!\Big[\ell_0(f_\theta,\bz)\Big] + \gamma\E_{\bz\sim\mathcal D}\!\Big[|u(\bz)|\Big] \\
    + \E_{\bz\sim\mathcal D}\!\Big[\lambda(\bz)\big(\ell(f_\theta,\bz)-c-u(\bz)\big)\Big].
\end{multline}
Thus, the dual function $g_{\gamma}(\lambda): L_1^+(\mathcal{\mathcal{D}}) \to \reals$  is
\[
g_{\gamma}(\lambda)
=
\inf_{f_\theta\in\mathcal H,\;u\in L_1(\mathcal Z)}
L(f_\theta,u,\lambda).
\]
Rearranging the terms in \eqref{eq:lagrangian_res}, we see that the minimizations over \(u\) and $f_\theta$ are decoupled,
\begin{multline}\label{eq:lagrangian_res_rear}
    \inf_{f_\theta\in\mathcal H,\;u\in L_1(\mathcal Z)} L(f_\theta,u,\lambda) = \\
    \inf_{f_\theta\in\mathcal H}
    \E_{\bz\sim\mathcal D_0}\!\Big[\ell_0(f_\theta,\bz)\Big] + \E_{\bz\sim\mathcal D}\Big[\lambda(\bz)\big(\ell(f_\theta,\bz)-c\big)\Big] \\ 
    + \inf_{\;u\in L_1(\mathcal Z)} \E_{\bz\sim\mathcal D}\!\Big[\gamma|u(\bz)|-\lambda(\bz)u(\bz)\Big].
\end{multline}
As the minimization over $u$ is separable in $\bz$,  it is enough to solve individually for each $u(\bz)$, 
\begin{equation}
\inf_{u(\bz)\in\mathbb R}
\left\{
\gamma|u(\bz)|-\lambda u(\bz)
\right\}
=
\begin{cases}
0, & |\lambda(\bz)| \leq \gamma,\\[1mm]
-\infty, & |\lambda(\bz)| > \gamma.
\end{cases}
\end{equation}
Therefore,
\begin{equation}
\inf_{u\in L_1(\mathcal Z)}
\E_{\bz\sim\mathcal D}\!\Big[\gamma|u(\bz)|-\lambda(\bz)u(\bz)\Big]
=
\begin{cases}
0,
& ||\lambda||_{\infty} \leq \gamma\\[1mm]
-\infty,
& \text{otherwise}.
\end{cases}
\label{eq:inf_u_res}
\end{equation}
Substituting~\eqref{eq:inf_u_res} into~\eqref{eq:lagrangian_res_rear}, we obtain
\[
g_{\gamma}(\lambda)
=
\begin{cases}
\displaystyle
g(\lambda)
& ||\lambda||_{\infty} \leq \gamma,
\\[3mm]
-\infty,
& \text{otherwise}.
\end{cases}
\]
Thus,
\begin{equation}
    \sup_{\lambda \in  L_1^+(\mathcal{\mathcal{D}})} g_{\gamma}(\lambda) = \sup_{\lambda \in \Lambda_\gamma} g(\lambda)
\end{equation}
which is exactly the claim of the Proposition.
\end{proof}

\section{Proof of Theorem \ref{thm:dual_learnability_res}} \label{apx:res_learnability}
Define the set of primal minimizers for any arbitrary $\lambda$ and $\hat \lambda$ as
\begin{equation}
	F_\theta^*(\lambda) = \arg \min_{f_\theta \in \mathcal{H}} L(\tilde f_\theta,\lambda)
	\;;\;
	\hat F_\theta (\hat{\lambda}) = \arg \min_{f_\theta \in \mathcal{H}} \hat{L}(\tilde f_\theta,\hat{\lambda}).
\end{equation}
for the Lagrangians defined in~\eqref{E:lagrangian_csl_param} and~\eqref{E:empirical_lagrangian} respectively. 

\noindent
\textbf{Existence of $\lambda_\gamma^*$ and $\hat \lambda_\gamma^*$}. 
In the case of  $\hat \lambda_\gamma^*$, for every fixed \(f_\theta\in\mathcal H\), the map
\begin{equation}
\lambda \mapsto \hat L(f_\theta,\lambda)
\end{equation}
is affine, and hence continuous. Therefore \(\hat g\),
being the pointwise infimum of continuous functions, is upper
semicontinuous~\cite[Lemma 2.41]{aliprantis2006infinite}.

The supremum is taken over $[0,\gamma]^N$, which  is compact by the Heine--Borel theorem. Therefore the restricted empirical dual problem \eqref{P:dual_res_emp} is the supremum of an upper semi-continuous function, and thus achieves its optimal solution.

In the case of $\lambda_\gamma^*$, the set
\begin{equation}
    \Lambda_\gamma
    :=
    \{\lambda\in L^\infty_+(\mathcal D): \|\lambda\|_\infty\le \gamma\}
\end{equation}

is  weak-\(*\) compact by the Banach--Alaoglu theorem, given the  weak-\(*\) topology
\(\sigma(L^\infty,L^1)\).

Moreover, for every
fixed \(f_\theta\in\mathcal H\), if \(\ell(f,\cdot)\in L^1(\mathcal D)\), then
\begin{equation}
    \lambda \mapsto L(f,\lambda)
    =
    \mathbb E_{\bz\sim\mathcal D_0}[\ell_0(f,\bz)]
    +
    \int \lambda(\bz)\ell(f,\bz)\,d\mathcal D(\bz)
\end{equation}
is weak-\(*\) continuous. Hence $g$
is weak-\(*\) upper semicontinuous. Thus, the supremum of a weak-\(*\) upper semicontinuous function over a weak-\(*\) compact set achieves its maximum.

\noindent
\textbf{Lower Bound.} 
By definition of the domain of \eqref{P:dual_res_emp}, $\hat \lambda^*_{\gamma} \in [0,\gamma]^N$. Thus, we can choose a $\tilde \lambda$ such that $||\tilde \lambda||_\infty \leq \gamma$ and
\begin{equation}\label{apx:def-dual-lower-res}
\tilde \lambda(\bz_n) = (\hat \lambda^*_{\gamma})_n.
\end{equation}

Let $\tilde f_{\theta} \in  F^*_\theta(\tilde \lambda)$. Sub-optimality of $\tilde \lambda$ yields
\begin{align}\label{eq:apx-low-bound-t1-0-res}
    D^*_{\gamma} - \hat L(\hat f_\theta^*,\hat \lambda_{\gamma}^*) &\geq  L(\tilde f_{\theta},\tilde \lambda) - \hat L(\hat f_\theta^*,\hat \lambda_{\gamma}^*).
\end{align}
whereas sub-optimality of $\tilde f_{\theta}$ yields
\begin{equation}
     L(\tilde f_{\theta},\tilde \lambda) - \hat L(\hat f_\theta^*,\hat \lambda_{\gamma}^*) \geq L(\tilde f_{\theta},\tilde \lambda) -\hat L( \tilde f_{\theta},\hat \lambda^*_{\gamma}).
\end{equation}
The Lagrangian difference is, by definition of the Lagrangians and equation \eqref{apx:def-dual-lower-res},
\begin{multline}\label{apx:lag_diffrences_up_bound-2-res}
L(\tilde f_\theta,\tilde \lambda) - \hat L(\tilde f_\theta,\hat \lambda^*_{\gamma}) = \\
\Big[ \E_{\bz \sim \mathcal D_0} \!\Big[ \ell_0\big(\tilde f_\theta, \bz \big) \Big] - \frac{1}{N_0} \sum_{n = 1}^{N_0} \ell_0\big(\tilde f_\theta,\bz_{n}\big) \Big] \\
+ \Big[ \E_{\bz \sim \mathcal{D}} [ \tilde\lambda(\bz) \ell\big(\tilde f_\theta, \bz \big)]  - \frac{1}{N}\sum_{n = 1}^{N} \tilde \lambda(\bz_n) \; \ell\big(\tilde f_\theta,\bz_n \big) \Big].
\end{multline}
Assumption \ref{ass:unif_conv} ensures that uniform convergence with respect to $\ell_0$ holds, and thus with probability $1-\delta$,
\begin{equation}\label{eq:apx-obj-bound-2-res}
\E_{\bz \sim \mathcal D_0} \!\Big[ \ell_0\big(\tilde f_\theta, \bz \big) \Big] - \frac{1}{N_0} \sum_{n = 1}^{N_0} \ell_0\big(\tilde f_\theta,\bz_{n}\big)   \geq -\zeta_0(N_0,\delta).
\end{equation}
For the second term in \eqref{apx:lag_diffrences_up_bound-2-res},  by construction $||\tilde \lambda||_\infty \leq \gamma$. Thus, we are under the conditions of Lemma \ref{lemma:rademacher_bound} and, with probability $1-\delta$,
\begin{align}\label{eq:apx-low-bound-t1-res}
 \E_{\bz \sim \mathcal{D}} [ \tilde\lambda(\bz) \ell\big(\tilde f_\theta, \bz \big)]  - \frac{1}{N}\sum_{n = 1}^{N} \tilde \lambda(\bz_n) \; \ell\big(\tilde f_\theta,\bz_n \big) \geq -\gamma \zeta(N,\delta).
\end{align}

\noindent
\textbf{Upper Bound.} By definition of the domain $\Lambda_\gamma$, we have that
\begin{equation}
    ||\lambda^*_{\gamma}||_{\infty} \leq \gamma.
\end{equation}
Let  $\tilde \lambda \in [0,\gamma]^N$ be the vector
\begin{equation}\label{apx:tilde-lambda-sampled}
    \tilde \lambda_n = \lambda^*_{\gamma}(\bz_n),
\end{equation}
and let $\tilde f_{\theta} \in  \hat F^*_\theta(\tilde \lambda)$.

Sub-optimality of $\tilde \lambda$ yields
\begin{align}\label{eq:apx-low-bound-t2-0-res}
	D_\gamma^* - \hat L(\hat f_\theta^*,\hat \lambda_{\gamma}^*) &\leq  D_\gamma^* - \hat L(\tilde f_\theta,\tilde \lambda),
\end{align}
whereas sub-optimality of $\tilde f_\theta$ yields
\begin{equation}
     D_\gamma^* - \hat L(\tilde f_\theta,\tilde \lambda) \leq L(\tilde f_\theta,\lambda^*_\gamma) - \hat L(\tilde f_\theta,\tilde \lambda).
\end{equation}
By definition of the Lagrangians and equation \eqref{apx:tilde-lambda-sampled},
\begin{multline}\label{apx:lag_diffrences_up_bound-2-res}
L(\tilde f_\theta,\lambda^*_\gamma) - \hat L(\tilde f_\theta,\tilde \lambda) = \\
\Big[ \E_{\bz \sim \mathcal D_0} \!\Big[ \ell_0\big(\tilde f_\theta, \bz \big) \Big] - \frac{1}{N_0} \sum_{n = 1}^{N_0} \ell_0\big(\tilde f_\theta,\bz_{n}\big) \Big] \\
+ \Big[ \E_{\bz \sim \mathcal{D}} [ \lambda_\gamma(\bz) \ell\big(\tilde f_\theta, \bz \big)]  - \frac{1}{N}\sum_{n = 1}^{N} \lambda_\gamma(\bz_n) \; \ell\big(\tilde f_\theta,\bz_n \big) \Big].
\end{multline}
Assumption \ref{ass:unif_conv} ensures that uniform convergence with respect to $\ell_0$ holds, and thus with probability $1-\delta$,
\begin{equation}\label{eq:apx-obj-bound-2-res}
\E_{\bz \sim \mathcal D_0} \!\Big[ \ell_0\big(\tilde f_\theta, \bz \big) \Big] - \frac{1}{N_0} \sum_{n = 1}^{N_0} \ell_0\big(\tilde f_\theta,\bz_{N}\big)   \leq \zeta_0(N_0,\delta).
\end{equation}
For the second term in \eqref{apx:lag_diffrences_up_bound-2-res},  as $||\tilde \lambda||_\infty \leq \gamma$ we are under the conditions of Lemma \ref{lemma:rademacher_bound} and, with probability $1-\delta$,
\begin{align}\label{eq:apx-low-bound-t2-res}
 \E_{\bz \sim \mathcal{D}} [ \lambda_\gamma(\bz) \ell\big(\tilde f_\theta, \bz \big)]  - \frac{1}{N}\sum_{n = 1}^{N} \tilde \lambda(\bz_n) \; \ell\big(\tilde f_\theta,\bz_n \big) \leq \gamma \zeta(N,\delta).
\end{align}

\noindent
\textbf{Union Bound}
Finally, combining the inequalities in \eqref{eq:apx-low-bound-t1-0-res}-\eqref{eq:apx-low-bound-t1-res}, the inequalities in \eqref{eq:apx-low-bound-t2-0-res}-\eqref{eq:apx-low-bound-t2-res} and applying a union bound argument, we get, with probability $1-5\delta$
\begin{equation}
    |D_\gamma^* - \hat L(\hat f_\theta^*, \hat \lambda^*_\gamma)|  \leq  2\zeta_0(N_0,\delta) + \gamma \zeta(N,\delta),
\end{equation}
which completes the proof.


\section{Duality in distribution space} \label{apx:distr_space}
The main text represents dual variables as functions
$\lambda\in L^1_+(\mathcal D)$, but the underlying objects are finite
nonnegative measures $P_\lambda$ on $\mathcal Z$. Whenever
$P_\lambda\ll\mathcal D$, the functional multiplier is the
Radon--Nikodym derivative
$\lambda=dP_\lambda/d\mathcal D$. The measure formulation also admits
singular multipliers and allows dual attainment to be established from
strict feasibility. Throughout this appendix, $\mathcal Z$ is compact,
$\ell_0(f_\theta,\cdot)$ and $\ell(f_\theta,\cdot)$ are continuous, and
$\mathcal M_+(\mathcal Z)$ denotes the cone of finite nonnegative Radon
measures equipped with the weak-* topology induced by
\begin{equation}
    \mathcal M(\mathcal Z)=C(\mathcal Z)^*.
    \label{eq:app-measure-duality}
\end{equation}

\subsection{The measure-valued dual problem}

For $P_\lambda\in\mathcal M_+(\mathcal Z)$, define
\begin{equation}
\begin{aligned}
    L_{\mathcal M}(f_\theta,P_\lambda)
    &:=
    \mathbb E_{\mathbf z\sim\mathcal D_0}
    \left[\ell_0(f_\theta,\mathbf z)\right]
    +
    \int_{\mathcal Z}
    \ell(f_\theta,\mathbf z)\,dP_\lambda(\mathbf z), \\
    g_{\mathcal M}(P_\lambda)
    &:=
    \inf_{f_\theta\in\mathcal H}
    L_{\mathcal M}(f_\theta,P_\lambda),
\end{aligned}
\label{eq:app-measure-lagrangian}
\end{equation}
and
\begin{equation}
    D_{\mathcal M}^\star
    :=
    \sup_{P_\lambda\in\mathcal M_+(\mathcal Z)}
    g_{\mathcal M}(P_\lambda).
    \label{eq:app-measure-dual}
\end{equation}

If $P_\lambda\ll\mathcal D$, the Radon--Nikodym theorem gives
\begin{equation}
    \lambda
    =
    \frac{dP_\lambda}{d\mathcal D}
    \in L^1_+(\mathcal D),
    \label{eq:app-rn-density-general}
\end{equation}
and the Lagrangian becomes
\begin{equation}
\begin{aligned}
    L_{\mathcal M}(f_\theta,P_\lambda)
    &=
    \mathbb E_{\mathbf z\sim\mathcal D_0}
    \left[\ell_0(f_\theta,\mathbf z)\right] \\
    &\quad+
    \mathbb E_{\mathbf z\sim\mathcal D}
    \left[
        \lambda(\mathbf z)\ell(f_\theta,\mathbf z)
    \right].
\end{aligned}
\label{eq:app-functional-recovery}
\end{equation}
Thus, the functional multiplier in the main text is the density of
the dual measure with respect to the data distribution. The measure
formulation is more general because it also permits singular measures,
such as Dirac masses.

Similarly, an empirical multiplier vector
$\hat\lambda\in\mathbb R^N_+$ induces the measure
\begin{equation}
    \hat P_{\hat\lambda}
    :=
    \frac{1}{N}
    \sum_{n=1}^N
    \hat\lambda_n\delta_{\mathbf z_n},
    \label{eq:app-empirical-measure}
\end{equation}
for which
\begin{equation}
    \int_{\mathcal Z}
    \ell(f_\theta,\mathbf z)\,
    d\hat P_{\hat\lambda}(\mathbf z)
    =
    \frac{1}{N}
    \sum_{n=1}^N
    \hat\lambda_n\ell(f_\theta,\mathbf z_n).
    \label{eq:app-empirical-integral}
\end{equation}

\subsection{Dual attainability from strict feasibility}

\begin{assumption}[Strict feasibility]
\label{ass:strict-feasible}
There exist $\bar f_\theta\in\mathcal H$ and $c>0$ such that
\begin{equation}
    \ell(\bar f_\theta,\mathbf z)
    \leq -c
    \qquad
    \text{for all }\mathbf z\in\mathcal Z.
    \label{eq:app-strict-feasibility}
\end{equation}
\end{assumption}

\begin{proposition}[Dual attainability]
\label{prop:measure-dual-attainability}
Under Assumption~\ref{ass:strict-feasible}, the supremum in
\eqref{eq:app-measure-dual} is attained by some
$P_\lambda^\star\in\mathcal M_+(\mathcal Z)$. Moreover,
\begin{equation}
    \left\|P_\lambda^\star\right\|_{\mathrm{TV}}
    \leq
    \frac{
        \mathbb E_{\mathbf z\sim\mathcal D_0}
        \left[\ell_0(\bar f_\theta,\mathbf z)\right]
        -
        D_{\mathcal M}^\star
    }{c}.
    \label{eq:app-tv-bound-statistical}
\end{equation}
The empirical dual problem also admits a maximizer.
\end{proposition}

\begin{proof}
For every $P_\lambda\in\mathcal M_+(\mathcal Z)$, strict feasibility
implies
\begin{equation}
\begin{aligned}
    g_{\mathcal M}(P_\lambda)
    &\leq
    L_{\mathcal M}(\bar f_\theta,P_\lambda) \\
    &\leq
    \mathbb E_{\mathbf z\sim\mathcal D_0}
    \left[\ell_0(\bar f_\theta,\mathbf z)\right]
    -
    c\left\|P_\lambda\right\|_{\mathrm{TV}}.
\end{aligned}
\label{eq:app-strict-feasibility-bound}
\end{equation}
Consequently, every maximizing sequence is bounded in total
variation. By the Banach--Alaoglu theorem, it has a weak-* convergent
subsequence.

For fixed $f_\theta$, the map
$P_\lambda\mapsto L_{\mathcal M}(f_\theta,P_\lambda)$ is weak-*
continuous because $\ell(f_\theta,\cdot)\in C(\mathcal Z)$. Therefore,
$g_{\mathcal M}$, as the infimum of weak-* continuous functions, is
weak-* upper semicontinuous. The limit of the convergent subsequence
thus attains $D_{\mathcal M}^\star$. Evaluating
\eqref{eq:app-strict-feasibility-bound} at the optimizer gives
\eqref{eq:app-tv-bound-statistical}.

The empirical result follows from the same argument in finite
dimension: strict feasibility bounds every maximizing sequence, and
upper semicontinuity guarantees attainment on the resulting compact
set.
\end{proof}

\subsection{Functional multipliers and sensitivity}

Suppose that the optimal measure satisfies
$P_\lambda^\star\ll\mathcal D$. Its Radon--Nikodym derivative
\begin{equation}
    \lambda^\star
    :=
    \frac{dP_\lambda^\star}{d\mathcal D}
    \in L^1_+(\mathcal D)
    \label{eq:app-optimal-rn-density}
\end{equation}
is the functional multiplier used in the main text.

\begin{lemma}[Bounded density implies domination]
\label{lemma:app-bounded-density}
If $\lambda^\star\in L^\infty(\mathcal D)$, then every measurable
set $A\subseteq\mathcal Z$ satisfies
\begin{equation}
    P_\lambda^\star(A)
    \leq
    \left\|\lambda^\star\right\|_\infty
    \mathcal D(A).
    \label{eq:app-domination}
\end{equation}
\end{lemma}

\begin{proof}
By the Radon--Nikodym representation,
\begin{equation}
\begin{aligned}
    P_\lambda^\star(A)
    &=
    \int_A
    \lambda^\star(\mathbf z)\,d\mathcal D(\mathbf z) \\
    &\leq
    \left\|\lambda^\star\right\|_\infty
    \mathcal D(A).
\end{aligned}
\end{equation}
\end{proof}

Hence, the condition
$\left\|\lambda^\star\right\|_\infty\leq\gamma$ means that the
sensitivity measure is dominated by $\gamma\mathcal D$:
\begin{equation}
    P_\lambda^\star(A)
    \leq
    \gamma\mathcal D(A).
    \label{eq:app-gamma-domination}
\end{equation}

To formalize the sensitivity interpretation, let
$u:\mathcal Z\to\mathbb R$ perturb the constraints as
$\ell(f_\theta,\mathbf z)\leq u(\mathbf z)$, and define
\begin{equation}
\begin{aligned}
    D_{\mathcal M}(u)
    :=
    \sup_{P_\lambda\in\mathcal M_+(\mathcal Z)}
    \inf_{f_\theta\in\mathcal H}
    \Bigg\{
        &\mathbb E_{\mathbf z\sim\mathcal D_0}
        \left[\ell_0(f_\theta,\mathbf z)\right] \\
        &+
        \int_{\mathcal Z}
        \left(
            \ell(f_\theta,\mathbf z)-u(\mathbf z)
        \right)
        dP_\lambda(\mathbf z)
    \Bigg\}.
\end{aligned}
\label{eq:app-dual-perturbation}
\end{equation}

\begin{lemma}[Dual sensitivity]
\label{lemma:app-dual-sensitivity}
If $P_\lambda^\star$ attains $D_{\mathcal M}(0)$, then
\begin{equation}
    D_{\mathcal M}(u)
    \geq
    D_{\mathcal M}(0)
    -
    \int_{\mathcal Z}
    u(\mathbf z)\,dP_\lambda^\star(\mathbf z).
    \label{eq:app-sensitivity}
\end{equation}
\end{lemma}

\begin{proof}
Evaluating the supremum defining $D_{\mathcal M}(u)$ at
$P_\lambda^\star$ gives
\begin{align}
    D_{\mathcal M}(u)
    &\geq
    \inf_{f_\theta\in\mathcal H}
    \Bigg\{
        \mathbb E_{\mathbf z\sim\mathcal D_0}
        \left[\ell_0(f_\theta,\mathbf z)\right]
        +
        \int_{\mathcal Z}
        \ell(f_\theta,\mathbf z)\,
        dP_\lambda^\star(\mathbf z)
    \Bigg\}
    \nonumber\\
    &\quad-
    \int_{\mathcal Z}
    u(\mathbf z)\,dP_\lambda^\star(\mathbf z)
    \nonumber\\
    &=
    D_{\mathcal M}(0)
    -
    \int_{\mathcal Z}
    u(\mathbf z)\,dP_\lambda^\star(\mathbf z).
\end{align}
\end{proof}

Therefore, $P_\lambda^\star(A)$ measures the sensitivity of the
optimal value to relaxing constraints on a region $A$. Because the
empirical problem enforces constraints only at sampled points, the
domination condition
$P_\lambda^\star(A)\leq\gamma\mathcal D(A)$ prevents regions with
small sampling probability from having disproportionately large
influence on the dual value. When domination fails, the dual measure
may concentrate on rarely sampled sets, and uniform convergence of
the losses alone need not control the dual approximation error.



\section{Proof of Theorem \ref{thm:pacc}}
\label{apx:pacc_convex}
\noindent
\textbf{Strong duality.}
By Assumptions~\ref{ass:convx-H} and~\ref{ass:strong-convx-loss}, \eqref{P:primal_statistical}  and \eqref{P:primal_empirical} are convex. By Assumption \ref{ass:dual_exis} the dual optimal variables are attained.\footnote{In standard results constraint qualifications such as Slater's condition are often used to guarantee the existence of optimal dual multipliers, which then yelds strong duality.} Therefore strong duality holds~\cite[Proposition 5.3.2]{bertsekas2009convex}, 
\begin{equation}\label{apx:strong-dual-pacc-proof}
    P^* = D^* \quad \; \text{and} \; \quad  \hat P^*  = \hat D^*. 
\end{equation}

\noindent
\textbf{Feasibility.} As Assumption~\ref{ass:unif_conv} holds and $\bar f_\theta$ is the solution of  \eqref{P:primal_empirical}, is immediate from Lemma~\ref{thm:empirical_is_relaxation} that, with probability $1-\delta$,
\begin{equation}\label{apx:pacc-proof-feas}
\Pr_{\bz\sim\mathcal D}
\big[
\ell(\bar f_\theta,\bz)\leq 0
\big]
\geq
1-\zeta_I .
\end{equation}

\noindent
\textbf{Upper bound.}
As Assumption~\ref{ass:unif_conv} holds and $\bar f_\theta$ is the solution of  \eqref{P:primal_empirical}, is immediate from Lemma~\ref{thm:empirical_is_relaxation} that, with probability $1-\delta$,
\begin{equation}\label{apx:pacc-proof-upper}
\E_{\bz \sim \mathcal D_0} \!\Big[ \ell_0\big(\bar f_{\theta},\bz \big) \Big]  - P^* \leq 2\zeta_0(N_0,\delta)
\end{equation}

\noindent
\textbf{Lower bound.}
By uniform convergence of the objective loss, with probability $1-\delta$,
\begin{align}\label{apx:proof-pacc-1}
\E_{\bz \sim \mathcal D_0}\!\Big[ \ell_0\big(\bar f_{\theta},\bz \big) \Big]   - P^* 
&\geq
\frac{1}{N_0}\sum_{n=1}^{N_0} \! \ell_0\big(\bar f_{\theta},\bz_n \big) -\zeta_0  - P^* \\
&= \hat{P}^* -\zeta_0  - P^*.
\end{align}

Using strong duality from \eqref{apx:strong-dual-pacc-proof}, we obtain
\begin{equation}
    \hat{P}^* -  P^* = \hat{D}^* -  D^*.
\end{equation}

Under Assumptions~\ref{ass:unif_conv}--\ref{ass:bounded-inf-norm}, dual learnability in Theorem~\ref{thm:dual-learn} gives, with probability at least $1-3\delta$,
\begin{equation}\label{apx:proof-pacc-2}
    \hat{D}^* -  D^* \geq -2\zeta_0(N_0,\delta) - \gamma\,\zeta(N,\delta).
\end{equation}

Combining \eqref{apx:proof-pacc-1}--\eqref{apx:proof-pacc-2} yields the lower bound, with probability $1-4\delta$,
\begin{equation}\label{apx:pacc-proof-lower}
   E_{\bz \sim \mathcal D_0}\!\Big[ \ell_0\big(\bar f_{\theta},\bz \big) \Big]   - P^*  \geq  -3\zeta_0(N_0,\delta) - \gamma\,\zeta(N,\delta)
\end{equation}

\noindent
\textbf{Union bound.}
Taking a union bound over \eqref{apx:pacc-proof-feas}, \eqref{apx:pacc-proof-upper} and  \eqref{apx:pacc-proof-lower}  yields the desired result with probability at least \(1-6\delta\).

\section{Proof of Proposition \ref{prop:primal-est-same-dual}}
\label{apx:pacc_dual}
As \eqref{P:primal_empirical} is strongly convex, it has a unique primal minimizer $\hat f_\theta^*$ that satisfy the first order optimality conditions~\cite[Proposition 5.3.2]{bertsekas2009convex},
\begin{equation}
\hat f_\theta^* \in \arg \min_{f_\theta \in \mathcal{H}} \hat L(f_\theta , \hat \lambda^*).    
\end{equation}

Hence, the primal solution minimizes the Lagrangian at the optimal dual multiplier. Since $\hat L(f_\theta , \hat \lambda^*)$ is strongly convex over $f_\theta$, this Lagrangian minimizer is unique, and thus it must coincide with $\bar f_\theta$. Therefore, the estimator $\bar f_\theta$ returned by \eqref{P:dual_empirical} solves  \eqref{P:primal_empirical}.


\section{Proof of Lemma \ref{lemma:dual-gap}} 
\label{apx:duality_gap_proof}
\noindent\textbf{Lower bound.} Weak duality immediately yields the lower bound,
\begin{equation}
    D^* \leq P^*.
\end{equation}



\noindent\textbf{Upper bound.} We consider the variational problem with constraints tightened by $M\nu$, explicitly,
\begin{equation}
\label{P:csl_variational_tightened}
\begin{aligned}
\widetilde{P}^\star_{-M{\nu}}  = \min_{\phi \in \hat{\mathcal{H}}}&\;
   \E_{(\bx,y) \sim \mathcal D_0} \!\Big[ \ell_0\big( \phi(\bx), y \big) \Big] \\
\text{s.to}\; &  \ell \big( \phi(\bx), y \big) 
    \leq  - M{\nu} \quad \mathcal{D} - \text{a.e.},
\end{aligned}
\end{equation}
Strong convexity of problem \eqref{P:csl_variational_tightened} combined with the dual attainability of Assumption~\ref{ass:constraint-qual-tightened} implies that problem~\eqref{P:csl_variational_tightened} is strongly dual, i.e.,
\begin{align}
\widetilde{P}^\star_{-M{\nu}}  = \widetilde{D}^\star_{-M{\nu}}
\end{align}
where $\widetilde{D}^\star_{-M{\nu}}$ denotes the value of the dual problem associated to~\eqref{P:csl_variational_tightened}.

Let $\tilde L_{M\nu}(\phi,\lambda)$ be the Lagrangian associated to~\eqref{P:csl_variational_tightened}, and note that 
\begin{equation}
\tilde L(\phi,\lambda) = \tilde L_{M\nu}(\phi, \lambda) - M\nu\E_{\mathcal{D}}\big[\lambda(\bz)\big]\quad  \forall \lambda\in L^1_+(\mathcal{D})
\end{equation}
holds by definition, where $\tilde L(\phi,\lambda)$ is the Lagrangian of \eqref{P:csl_variational}. Let $\lambda^\star_{M\nu}\in L^1_+(\mathcal{D})$ denote a solution to $\widetilde{D}^\star_{-M{\nu}}$. Sub-optimality of $\lambda^\star_{M\nu}$ for \eqref{P:dual_statistical}  yields
\begin{align}\label{eq:apx-pert-bound-1}
    D^\star \geq \min_{f_\theta \in \mathcal H} L\left(f_\theta, \lambda^\star_{M\nu}\right) .
\end{align}
Moreover, as $\mathcal{H}\subseteq \hat{\mathcal{H}}$, 
\begin{align}
    \min_{f_\theta \in \mathcal{H}} L\left(f_\theta, \lambda^\star_{M\nu}\right) & \geq \min_{\phi \in \bar{\mathcal{H}}} L(\phi, \lambda^\star_{M\nu}) \\
     & = \min_{\phi \in \bar{\mathcal{H}}} L_{M\nu}(\phi, \lambda^\star_{M\nu}) - M\nu\E_{\mathcal{D}}\big[\lambda^\star_{M\nu}(\bz)\big] \label{apx:lag_} \\
    & = \widetilde{D}^\star_{-M{\nu}} - M\nu\E_{\mathcal{D}}\big[\lambda^\star_{M\nu}(\bz)\big] \\
    & = \widetilde{P}^\star_{-M{\nu}} - M\nu\E_{\mathcal{D}}\big[\lambda^\star_{M\nu}(\bz)\big] \label{eq:apx-pert-err-2}
\end{align}
where the last equality holds due to strong duality of $\widetilde{P}^\star_{-M{\nu}} $. Combining~\eqref{eq:apx-pert-bound-1}-\eqref{eq:apx-pert-err-2} yields
\begin{align}\label{eqn:ineq_param_dual_Mnuvar}
    D^\star \geq \widetilde{P}^\star_{-M{\nu}} - M\nu\E_{\mathcal{D}}\big[\lambda^\star_{M\nu}(\bz)\big].
\end{align}
Let $\phi_{M\nu} \in \hat{\mathcal{H}}$ denote a solution to ~\eqref{P:csl_variational_tightened}.
From Assumption~\ref{ass:near-univ} there exists $f_{\theta_{M\nu}}\in\mathcal{H}$ such that
\begin{align}\label{eqn:constraint_approx}
    |\ell \big( \phi_{M\nu},\bz \big) - \ell \big( f_{\theta_{M\nu}},\bz \big)|
    \leq  M{\nu} \quad \mathcal{D} - \text{a.e.}
\end{align}
Given that $\phi_{M\nu}$ is feasible for~\eqref{P:csl_variational_tightened}, then \eqref{eqn:constraint_approx} implies that
\begin{align}
     \ell \big( f_{\theta_{M\nu}}, \bz \big) \leq 0,
\end{align}
that is, $f_{\theta_{M\nu}}$ is feasible for \eqref{P:primal_statistical}.
Therefore, by sub-optimality of $\theta_{M\nu}$
\begin{align}
   P^\star \leq \mathbb{E}_{\mathcal{D}_0}[\ell_0(f_{\theta_{M\nu}}, \bz)] .
\end{align}
Thus,
\begin{align}
  \widetilde{P}^\star_{-M{\nu}} & \geq \widetilde{P}^\star_{-M{\nu}}  + P^\star-\mathbb{E}_{\mathcal{D}}[\ell_0(f_{\theta_{M\nu}}, \bz)]  \\
  & = P^\star+ \mathbb{E}_{\mathcal{D}_0}[\ell_0(\phi_{M\nu}, \bz)-\ell_0(f_{\theta_{M\nu}}, \bz)]\\
  & \geq P^\star - M\nu_0,\label{eqn:ineq_pstar_pnu}
\end{align}
where the last inequality is due to Assumption~\ref{ass:near-univ}.

Finally, combining~\eqref{eqn:ineq_param_dual_Mnuvar} and~\eqref{eqn:ineq_pstar_pnu} yields the desired bound.

\section{Proof of Theorem \ref{thm:near-pacc}} 
\label{apx:nearPACC}
\noindent
\textbf{Feasibility.} As Assumption~\ref{ass:unif_conv} holds and $\bar f_\theta$ is the solution of  \eqref{P:primal_empirical}, is immediate from Lemma~\ref{thm:empirical_is_relaxation} that, with probability $1-\delta$,
\begin{equation}\label{apx:near-pacc-proof-feas}
\Pr_{\bz\sim\mathcal D}
\big[
\ell(\bar f_\theta,\bz)\leq 0
\big]
\geq
1-\zeta_I .
\end{equation}

\noindent
\textbf{Upper bound.}
As Assumption~\ref{ass:unif_conv} holds and $\bar f_\theta$ is the solution of  \eqref{P:primal_empirical}, is immediate from Lemma~\ref{thm:empirical_is_relaxation} that, with probability $1-\delta$,
\begin{equation}\label{apx:near-pacc-proof-upper}
\E_{\bz \sim \mathcal D_0} \!\Big[ \ell_0\big(\bar f_{\theta},\bz \big) \Big]  - P^*\leq 2\zeta_0(N,\delta)
\end{equation}

\noindent
\textbf{Lower bound.}
By uniform convergence of the objective loss, with probability $1-\delta$,
\begin{align}\label{apx:near-proof-pacc-1}
\E_{\bz \sim \mathcal D_0}\!\Big[ \ell_0\big(\bar f_{\theta},\bz \big) \Big]   - P^* 
&\geq
\sum_{\bz_n \in S_0} \! \ell_0\big(\bar f_{\theta},\bz_n \big) -\zeta_0  - P^* \\
&= \hat{P}^* -\zeta_0  - P^*.
\end{align}

Using weak duality, we obtain
\begin{equation}
    \hat{P}^* \geq \hat{D}^*.
\end{equation}

We decompose
\begin{equation}
\hat D^\star-P^\star
=
(\hat D^\star-D^\star)
+
(D^\star-P^\star).
\end{equation}

By Theorem~\ref{thm:dual-learn}, with probability at least \(1-2\delta\),
\begin{equation}
    \hat D^\star-D^\star
    \geq
    -2\zeta_0-\gamma\zeta .
\end{equation}
By Lemma~\ref{lemma:dual-gap},
\begin{equation}\label{apx:near-proof-pacc-2}
    D^\star-P^\star
    \geq
    -M\nu
    \E_{\mathcal D}
    \big[
        \lambda^\star_{M\nu}(\bz)
    \big]
    -
    M\nu_0
    =
    -\Delta_{\mathrm{gap}} .
\end{equation}

Combining \eqref{apx:near-proof-pacc-1}--\eqref{apx:near-proof-pacc-2} yields the lower bound, with probability $1-3\delta$,
\begin{equation}\label{apx:near-pacc-proof-lower}
   E_{\bz \sim \mathcal D_0}\!\Big[ \ell_0\big(\bar f_{\theta},\bz \big) \Big]   - P^*  \geq  -\Delta_{\mathrm{gap}} -3\zeta_0(N,\delta) - \gamma\,\zeta(N,\delta).
\end{equation}

\noindent
\textbf{Union bound.}
Taking a union bound over \eqref{apx:near-pacc-proof-feas}, \eqref{apx:near-pacc-proof-upper} and  \eqref{apx:near-pacc-proof-lower}  yields the desired result with probability at least \(1-5\delta\). 

\section{Proof of Proposition \ref{prop:dual_near_pacc}} 
\label{apx:nearPACC-dual}
\noindent
\textbf{Feasibility.}
We first show by contradiction that exists at least one primal minimizer that is feasible for \eqref{P:primal_empirical}, and then use uniform convergence to extend to PACC feasibility of \eqref{P:primal_statistical}.

A vector \(v \in \reals^N\) is a supergradient of \(\hat g\) at \(\lambda\) if
\begin{equation}
    g(\lambda') \leq g(\lambda) + \langle v,\lambda'-\lambda\rangle
    \qquad \forall \lambda' .
\end{equation}
The superdifferential \(\partial g(\lambda)\) is defined as the set of all supergradients of \(g\) at \(\lambda\). Since \(\hat\lambda^\star\) maximizes the concave function \(\hat g\) over \(\mathbb{R}_+^N\), the first-order optimality condition is
\begin{equation}\label{eq:apx-first-ord-sup}
    0\in\partial \hat g(\hat\lambda^\star).
\end{equation}

Let
\begin{equation} 
    \hat{\mathcal F}(\hat\lambda^*)
    :=
    \arg\min_{f_\theta\in\mathcal H}
    \hat L(f_\theta,\hat\lambda^*)
\end{equation}

be the set of all primal minimizers of the empirical Lagrangian for \(\hat\lambda^*\). 

For each \(f_\theta\in \hat{\mathcal F}(\hat\lambda^*)\), define its empirical slack vector
\begin{equation}
    s(f_\theta)
    :=
    \big(
        [\ell(f_\theta,\bz_1)]_+,
        \ldots,
        [\ell(f_\theta,\bz_N)]_+
    \big)
    \in \mathbb R_+^N .
\end{equation}
where $ [\ell(f_\theta,\bz_n)]_+ = \max \{0,\ell(f_\theta,\bz_n)\}$.
By Danskin's theorem \cite[Theorem 2.87]{ruszczynski2011nonlinear}, the superdifferential of the empirical dual function at \(\hat\lambda^*\) is given by

\begin{equation}
    \partial \hat g(\hat\lambda^*)
    =
    \operatorname{conv}
    \Big(
        \big\{
            s(f_\theta):
            f_\theta\in \hat{\mathcal F}(\hat\lambda^*)
        \big\}
    \Big).
\end{equation}

Suppose, by contradiction, that every element of \(\hat{\mathcal F}(\hat\lambda^*)\) violates at least one constraint. Then, for every \(f_\theta\in \hat{\mathcal F}(\hat\lambda^*)\), there exists an index \(n\in\{1,\ldots,N\}\) such that
\begin{equation}
    \ell(f_\theta,\bz_n) > 0.
\end{equation}

Equivalently, every slack vector associated with a primal minimizer has a positive component.

By the superdifferential characterization in \ref{eq:apx-first-ord-sup}, this implies that \(0\notin \partial \hat g(\hat\lambda^*)\). This contradicts the optimality of \(\hat\lambda^*\) for the concave maximization problem~\eqref{P:dual_empirical}. Hence, there exists at least one
\begin{equation}
    \bar f_\theta
    \in
    \arg\min_{f_\theta\in\mathcal H}
    \hat L(f_\theta,\hat\lambda^*)
\end{equation}

such that
\begin{equation}
    \ell(\bar f_\theta,\bz_n)\leq 0,
    \qquad
    n=1,\ldots,N.
\end{equation}

Thus, \(\bar f_\theta\) is feasible for~\eqref{P:primal_empirical}.

Since \(\bar f_\theta\) satisfies the constraint on all samples,
\[
    \frac{1}{N}
    \sum_{n=1}^N
    \mathbb I\{\ell(\bar f_\theta,\bz_n)\leq 0\}
    =
    1 .
\]
By Assumption~\ref{ass:unif_conv}, uniform convergence holds for the constraint indicator \(\mathbb I\{\ell\leq 0\}\). Therefore, with probability at least \(1-\delta\),
\[
\left|
\Pr_{\bz\sim\mathcal D}
\big[
    \ell(\bar f_\theta,\bz)\leq 0
\big]
-
\frac{1}{N}
\sum_{n=1}^N
\mathbb I\{\ell(\bar f_\theta,\bz_n)\leq 0\}
\right|
\leq
\zeta_I .
\]
Using the empirical feasibility identity above gives
\[
    \Pr_{\bz\sim\mathcal D}
    \big[
        \ell(\bar f_\theta,\bz)\leq 0
    \big]
    \geq
    1-\zeta_I .
\]

\noindent
\textbf{Upper bound.}
By uniform convergence of the objective loss, with probability at least \(1-\delta\),
\begin{align}\label{apx:near-proof-dual-1}
\E_{\bz \sim \mathcal D_0}
\!\Big[
    \ell_0\big(\bar f_{\theta},\bz \big)
\Big]
-
P^\star 
&\leq
\frac{1}{N_0}
\sum_{n=1}^{N_0}
\ell_0\big(\bar f_{\theta},\bz_n \big)
+
\zeta_0
-
P^\star .
\end{align}
Since \(\bar f_\theta \in \arg\min_{f_\theta\in\mathcal H}\hat L(f_\theta,\hat\lambda^\star)\), we have
\begin{equation}
    \hat D^\star
    =
    \hat L(\bar f_\theta,\hat\lambda^\star)
    =
    \frac{1}{N_0}
    \sum_{n = 1}^{N_0}
    \ell_0\big(\bar f_{\theta},\bz_n \big)
    +
    \frac{1}{N}
    \sum_{n=1}^N
    \hat \lambda^\star_n
    \ell(\bar f_\theta,\bz_n).
\end{equation}
Equivalently, using
\[
\Delta_{\mathrm{slack}}
=
-\frac{1}{N}
\sum_{n=1}^N
\hat \lambda^\star_n
\ell(\bar f_\theta,\bz_n),
\]
we obtain
\begin{equation}\label{apx:near-proof-dual-slack-id}
    \frac{1}{N_0}
    \sum_{n=1}^{N_0}
    \ell_0\big(\bar f_{\theta},\bz_n \big)
    =
    \hat D^\star
    +
    \Delta_{\mathrm{slack}} .
\end{equation}
Substituting~\eqref{apx:near-proof-dual-slack-id} into~\eqref{apx:near-proof-dual-1} gives
\begin{equation}
\E_{\bz \sim \mathcal D_0}
\!\Big[
    \ell_0\big(\bar f_{\theta},\bz \big)
\Big]
-
P^\star
\leq
\hat D^\star
-
P^\star
+
\Delta_{\mathrm{slack}}
+
\zeta_0 .
\end{equation}
We decompose
\[
\hat D^\star - P^\star
=
(\hat D^\star-D^\star)
+
(D^\star-P^\star).
\]
By Theorem~\ref{thm:dual-learn}, with probability at least \(1-2\delta\),
\begin{equation}
    \hat D^\star-D^\star
    \leq
    2\zeta_0+\gamma\zeta .
\end{equation}
Moreover, by weak duality,
\begin{equation}\label{apx:near-proof-dual-2}
    D^\star-P^\star \leq 0 .
\end{equation}
Combining the last four equations yields
\begin{equation}\label{apx:pacc-dual-upper}
\E_{\bz \sim \mathcal D_0}
\!\Big[
    \ell_0\big(\bar f_{\theta},\bz \big)
\Big]
-
P^\star
\leq
\Delta_{\mathrm{slack}}
+
3\zeta_0
+
\gamma\zeta .
\end{equation}

\noindent
\textbf{Lower bound.}
By uniform convergence of the objective loss, with probability at least \(1-\delta\),
\begin{align}\label{apx:near-proof-pacc-3}
\E_{\bz \sim \mathcal D_0}
\!\Big[
    \ell_0\big(\bar f_{\theta},\bz \big)
\Big]
-
P^\star 
&\geq
\frac{1}{N_0}
\sum_{n=1}^{N_0}
\ell_0\big(\bar f_{\theta},\bz_n \big)
-
\zeta_0
-
P^\star .
\end{align}
Since \(\bar f_\theta\) is feasible for~\eqref{P:primal_empirical} and \(\hat\lambda^\star\in\mathbb R_+^N\),
\begin{equation}
   \frac{1}{N}
   \sum_{n=1}^N
   \hat \lambda^\star_n
   \ell(\bar f_\theta,\bz_n)
   \leq
   0 .
\end{equation}
Therefore,
\begin{align}
\hat D^\star &= \frac{1}{N_0}
\sum_{n=1}^{N_0}
\ell_0\big(\bar f_{\theta},\bz_n \big)
+
\frac{1}{N}
\sum_{n=1}^N
\hat \lambda^\star_n
\ell(\bar f_\theta,\bz_n)
\\
&\leq
\frac{1}{N_0}
\sum_{n=1}^{N_0}
\ell_0\big(\bar f_{\theta},\bz_n \big).
\end{align}
Substituting this inequality into~\eqref{apx:near-proof-pacc-3} gives
\begin{equation}
\E_{\bz \sim \mathcal D_0}
\!\Big[
    \ell_0\big(\bar f_{\theta},\bz \big)
\Big]
-
P^\star
\geq
\hat D^\star
-
P^\star
-
\zeta_0 .
\end{equation}
Again,
\begin{equation}
\hat D^\star-P^\star
=
(\hat D^\star-D^\star)
+
(D^\star-P^\star).
\end{equation}

By Theorem~\ref{thm:dual-learn}, with probability at least \(1-2\delta\),
\begin{equation}
    \hat D^\star-D^\star
    \geq
    -2\zeta_0-\gamma\zeta .
\end{equation}
By Lemma~\ref{lemma:dual-gap},
\begin{equation}\label{apx:near-proof-pacc-4}
    D^\star-P^\star
    \geq
    -M\nu
    \E_{\mathcal D}
    \big[
        \lambda^\star_{M\nu}(\bz)
    \big]
    -
    M\nu_0
    =
    -\Delta_{\mathrm{gap}} .
\end{equation}
Combining the last four equations yields
\begin{equation}\label{apx:pacc-dual-lower}
\E_{\bz \sim \mathcal D_0}
\!\Big[
    \ell_0\big(\bar f_{\theta},\bz \big)
\Big]
-
P^\star
\geq
-\Delta_{\mathrm{gap}}
-
3\zeta_0
-
\gamma\zeta .
\end{equation}

\noindent
\textbf{Union bound.}
Combining~\eqref{apx:pacc-dual-upper} and~\eqref{apx:pacc-dual-lower}, we obtain
\begin{equation}
\left|
\E_{\bz \sim \mathcal D_0}
\!\Big[
    \ell_0\big(\bar f_{\theta},\bz \big)
\Big]
-
P^\star
\right|
\leq
\max\{\Delta_{\mathrm{slack}},\Delta_{\mathrm{gap}}\}
+
3\zeta_0
+
\gamma\zeta .
\end{equation}
Together with the feasibility bound
\begin{equation}
\Pr_{\bz\sim\mathcal D}
\big[
    \ell(\bar f_\theta,\bz)\leq 0
\big]
\geq
1-\zeta_I,
\end{equation}

and applying a union bound over the upper-bound, lower-bound, dual-learnability, and feasibility events, the stated guarantees hold simultaneously with probability at least \(1-5\delta\).


\end{document}